\documentclass[AMA,STIX1COL]{WileyNJD-v2}

\usepackage{graphicx}
\usepackage{amsmath}
\usepackage{amsthm}
\usepackage{caption}
\usepackage{subcaption}

\newtheorem{conjuncture}{Conjuncture}

\articletype{Original Article}%

\begin{document}

\title{Deep Q-learning: a robust control approach}

\author[1]{Balázs Varga*}

\author[1]{Balázs Kulcsár}

\author[2]{Morteza Haghir Chehreghani}

\authormark{B. Varga \textsc{et. al.}}

\address[1]{\orgdiv{Department of Electrical Engineering}, \orgname{Chalmers University of Technology}, \orgaddress{\street{Hörsalsvägen 11}, \city{Gothenburg}, \country{Sweden}}}

\address[2]{\orgdiv{Department of Computer Science and Engineering}, \orgname{Chalmers University of Technology}, \orgaddress{\street{Hörsalsvägen 11}, \city{Gothenburg}, \country{Sweden}}}

\corres{*Balázs Varga. \email{balazsv@chalmers.se}}


\abstract[Summary]{This work aims at constructing a bridge between robust control theory and reinforcement learning. Although, reinforcement learning has shown admirable results in complex control tasks, the agent's learning behaviour is opaque. Meanwhile, system theory has several tools for analyzing and controlling dynamical systems. 
This paper places deep Q-learning is into a control-oriented perspective to study its learning dynamics with well-established techniques from robust control. An uncertain linear time-invariant model is formulated by means of the neural tangent kernel to describe learning. This novel approach allows giving conditions for stability (convergence) of the learning and enables the analysis of the agent's behaviour in frequency-domain. The control-oriented approach makes it possible to formulate robust controllers that inject dynamical rewards as control input in the loss function to achieve better convergence properties. Three output-feedback controllers are synthesized: gain scheduling $\mathcal{H}_2$, dynamical $\mathcal{H}_\infty$, and fixed-structure $\mathcal{H}_\infty$ controllers. Compared to traditional deep Q-learning techniques, which involve several heuristics, setting up the learning agent with a control-oriented tuning methodology is more transparent and has well-established literature. The proposed approach does not use a target network and randomized replay memory. The role of the target network is overtaken by the control input, which also exploits the temporal dependency of samples (opposed to a randomized memory buffer).  Numerical simulations in different OpenAI Gym environments suggest that the $\mathcal{H}_\infty$ controlled learning can converge faster and receive higher scores (depending on the environment) compared to the benchmark Double deep Q-learning.}


\jnlcitation{\cname{%
\author{B. Varga}, 
\author{B. Kulcsár}, and
\author{M. H.  Chehreghani}} (\cyear{2022}), 
\ctitle{Deep Q-learning: a robust control approach}, \cjournal{International Journal of Robust and Nonlinear Control}}

\maketitle

\footnotetext{\textbf{Abbreviations:} 
DDQN, Double Deep Q-Learning;
DQN, Deep Q-Network;
LFT, Linear Fractional Transformation; 
LPV, Linear Parameter Varying;
LQ, Linear Quadratic;
LTI, Linear Time-Invariant;
MDP, Markov Decision Process;
ML, Machine Learning;
NTK, Neural Tangent Kernel;
NN, Neural Network;
RL, Reinforcement Learning
}

\section{Introduction}
In the past decade, the success of neural networks (NNs) in various approximation and regression tasks has led to significant uptake of machine learning (ML) methods in various areas of science and real-world applications. On the other hand, working with large data sets, the black-box nature, and complex structure of these function approximators often hamper in-depth human understanding of such methods. Consequently, efforts have been made to improve the transparency of machine learning both in terms of training an ML model and the results produced by the trained model \cite{adadi2018peeking, roscher2020explainable}. Additionally, making such heuristic learning algorithms converge requires tweaking and experimenting. 

Although machine learning-based controllers  often outperform classical control, especially in highly nonlinear environments, their stability and performance are seldom guaranteed analytically \cite{hoel2018automated, zhou2021model}.
Control theory has a well-established and mathematically sound toolkit to analyze dynamical systems and synthesize stabilizing, robust controllers \cite{skogestad2007multivariable}.
This paper focuses on the control theory-based analysis of reinforcement learning (RL). Earlier, some authors dealt with connecting RL with with classical control. \cite{bradtke1994adaptive} shows that dynamic programming based reinforcement learning (Q-learning, in particular) converges to an optimal linear quadratic (LQ) regulator if the environment is a linear system. On the other hand, RL shines in complex environments where formulating a closed-form solution is impossible. Several works deal with synergized model-based and data driven controllers to improve the performance of the controlled process \cite{kretchmar2001robust, hegedus2020handling, wang2020approximate} or analyze learned controllers with tools from control \cite{donti2020enforcing, perrusquia2020robust, liu2020h, wang2022intelligent}. Meanwhile, control theory is seldom utilized to enhance the agent's training performance.

This work is motivated by the lack of convergence guarantees in deep Q-learning \cite{sutton2018reinforcement, van2018deep}.
Some recent advances in DQN modify the temporal difference target in order to achieve better convergence results, e.g.~\cite{pohlen2018observe, durugkar2018td}. \cite{achiam2019towards} aims at characterizing divergence in deep Q-learning with the help of the recently introduced neural tangent kernel (NTK, \cite{jacot2018neural}). In addition, they propose an algorithm that scales the learning rate to ensure convergence. In \cite{ohnishi2019constrained} an additive regularization term is used to constrain the loss and enhance convergence. Yet, the of majority of deep Q-learning applications employ some heuristics such as a target network \cite{mnih2015human} or random experience replay \cite{carvalho2020new}.

Reformulating learning as a dynamical system poses an opportunity to further study its divergent nature and formulate stabilizing controllers to improve learning performance.

\textbf{Contribution.} This work aims at constructing a bridge between robust control theory and reinforcement learning. To this end, techniques from robust control theory are borrowed to compensate the non-convergent behaviour of deep Q-learning via cascade control. Instead of introducing additional control to the agent-environment interaction, the dynamics of learning (the temporal evolution of the Q-function) is studied. 
First, learning is embedded into a state-space framework as an uncertain, linear, time-invariant (LTI) system through the NTK. Based on the dynamical system description, convergence (or stability) can be concluded in a straightforward way. As opposed to \cite{achiam2019towards}, stability is ensured via modifying the temporal difference term via robust stabilizing controllers injecting fictitious rewards. In this paper, three controllers are synthesize and benchmarked: static output-feedback gain scheduling $\mathcal{H}_2$, dynamic $\mathcal{H}_\infty$, and fixed-structure $\mathcal{H}_\infty$ controllers. The primary motivation for robust control is that it is capable of taking into account the uncertain nature of a reinforcement learning problem. In addition, the NTK does not have to be recomputed in every step; its variation can be included as a parametric uncertainty in the controller design. This yields a computationally more efficient methodology than the one proposed in \cite{achiam2019towards}.
The proposed control-oriented approach makes parameter tuning more straightforward and transparent (i.e.~involving fewer heuristics). The two aforementioned common heuristics of deep Q-learning (target network and random experience replay, \cite{carvalho2020new}) are not needed. Instead, the temporal dependency of samples is exploited through the dynamical system formulation. Robust control can support the learning process, making it more explainable. Results suggest that robust controlled learning performs on par with DDQN in the benchmark environments. 

The structure of the paper follows the MAD, {\textit{Modeling}}, {\textit{Analysis}}, and {\textit{Design}} framework. After the preliminaries (Section \ref{sec:Prelim}), the learning dynamics of Q-learning is formulated as an uncertain LTI system (Section \ref{sec:DQN_modeling}). Then, based on the formulated model, three controllers are formulated: Section \ref{sec:LQ_control} formulates an $\mathcal{H}_2$ output-feedback control, in Section \ref{sec:Hinf_control} a dynamical $\mathcal{H}_\infty$ controller is synthesized in frequency domain. Then, in Section \ref{sec:Hinf_fixed_control} the $\mathcal{H}_\infty$ controller design is adjusted to result in a controller with gains akin to the result of the $\mathcal{H}_2$ controller design. The proposed controlled learning approaches are thoroughly analyzed and compared in three challenging OpenAI Gym environments: Cartpole, Acrobot, and Mountain car (Section \ref{sec:experiments}). Finally, Section \ref{sec:conclusions} concludes the findings of this paper.

\section{Preliminaries}
\label{sec:Prelim}
This section consists of two parts. In the first part, deep Q-learning alongside its nomenclature is given. Second, the Neural Tangent Kernel (NTK, \cite{jacot2018neural}) is introduced, which can be used to describe the evolution
of neural networks under gradient descent in function space.

\subsection{Deep Q-learning}
In reinforcement learning, (in the absence of labeled data), the agent learns in a trial and error way, interacting with its environment. The learning agent faces a sequential decision problem and receives feedback as a performance measure \cite{sutton2018reinforcement}. This interaction is commonly depicted as the feedback structure in Figure \ref{fig:RL_setting}.
\begin{figure}[htb!]
\centering
\includegraphics[width=0.2\linewidth]{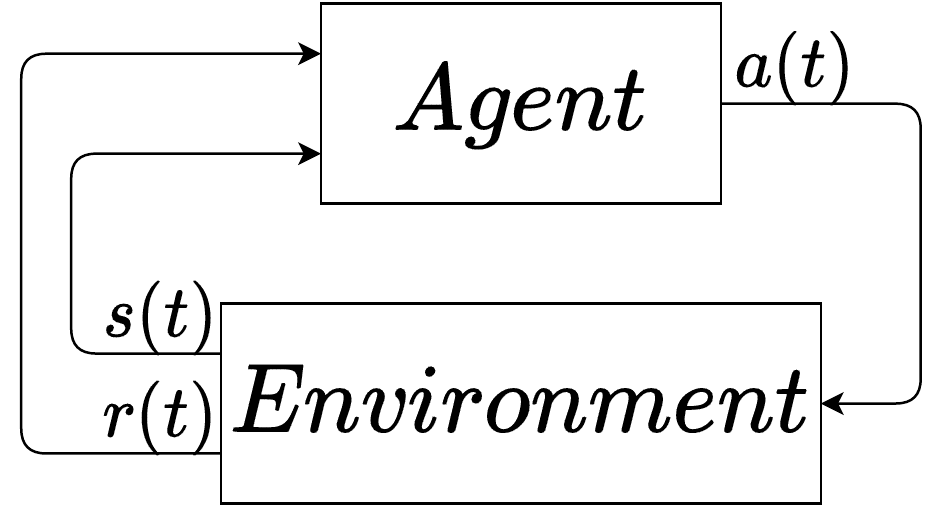}
\caption{Agent–environment interaction in reinforcement learning}
\label{fig:RL_setting}
\end{figure}
This sequential decision problem can be described with a (discrete) Markov Decision Process (MDP) characterized by the following 5-tuple: $(\mathcal{S}, \mathcal{A}, \mathcal{P}_T, R, \gamma)$, where $\mathcal{S} \subseteq \mathbb R^{n_s}$ is the continuous state-space with $n_s$ dimensions. $\mathcal{A} \subset \mathbb{Z}$ is the finite, discrete action space, $\mathcal{P}_T$ is the transition probability matrix, $R \in \mathbb{R}$ is the reward accumulated by the agent, and $\gamma \in ]0,1]$ is the discount factor. The agent traverses the MDP following policy $\pi(a(t)|s(t))$ with discrete time-step $t$. Reinforcement learning methods compute a mapping from the set of states
of the environment to the set of possible actions in order to maximize the expected discounted cumulative reward. 

One common way to tackle an RL problem is Q-learning. Here, the aim is learning the state-action-value (or Q) function - the measure of the overall expected reward for taking action $a(t)$ at state $s(t)$
\begin{equation}
\label{eq:Qdef}
    Q(s(t),a(t)) = \mathbb{E}_\pi \left( \sum_{\tau=0}^{\infty} \gamma^\tau r(t+\tau+1) | s(t), a(t)) \right),
\end{equation}
with $r(t) \in R$ being the immediate reward. In Q-learning the states and actions are discretized and can have huge cardinality. Thus, it suffers from the curse of dimensionality. Deep Q-learning alleviates this problem via approximating the Q-function with a neural network (deep Q-network, DQN). Thus, the Q-function takes an $n_s$ dimensional environment state $s(t) \in \mathcal{S}$ and evaluates the Q-value for action $a(t) \in \mathcal{A}$, $Q: \mathcal{S} \times \mathcal{A} \rightarrow \mathbb{R}$. Then, the policy $\pi(a(t)|s(t))$ selects the action corresponding to the largest Q-value in an $\varepsilon$-greedy way. 
Deep Q-learning learns by minimizing the temporal difference (1-step estimation error) at time $t$ following the quadratic loss function (mean squared Bellman residual \cite{baird1995residual}):
\begin{equation}
\label{eq:loss_0}
 \mathcal L(t) = \frac{1}{2}(\underbrace{r(t) + \gamma \underset{a \in \mathcal{A}}{\mathrm{max}} Q(s'(t),a, \theta(t))}_{target} - \underbrace{Q(s(t), a(t), \theta(t))}_{actual})^2,
\end{equation}
with $s'(t) = s(t+1)$ being the next state. 
Then, with learning rate $\alpha$, the weights of the neural network via gradient descent is
 \begin{align}
\label{eq:theta_update} 
    \theta(t+1)  = \theta(t) - \alpha \frac{\partial \mathcal L(t)}{\partial \theta(t)} = 
    \theta(t) + \alpha (r(t) + \gamma \underset{a \in \mathcal{A}}{\mathrm{max}} Q(s'(t),a, \theta(t)) - Q(s(t), a(t), \theta(t))) \frac{\partial Q(s(t), a(t), \theta(t))}{\partial \theta(t)}^T.
\end{align}
The gradient (assuming the Q-function is differentiable) $\frac{\partial Q(s(t), a(t), \theta(t))}{\partial \theta(t)}^T$ determines the “direction” in
which this update is performed. Observe that the target value $r(t) + \gamma \underset{a \in \mathcal{A}}{\mathrm{max}} Q(s'(t),a, \theta(t))$ also depends on $\theta(t)$. Thus, the correct gradient would be $\gamma \left(\frac{\partial }{\partial \theta(t)}\underset{a \in \mathcal{A}}{\mathrm{max}} Q(s'(t),a, \theta(t))\right)^T - \left(\frac{\partial }{\partial \theta(t)}Q(s(t), a(t), \theta(t))\right)^T$. On the other hand, the mainstream Q-learning algorithms perform the TD update with $\frac{\partial Q(s(t), a(t), \theta(t))}{\partial \theta(t)}^T$, resulting in faster and more stable algorithms \cite{baird1995residual}. In the sequel, this more common approach will be adhered to. 

Deep Q-learning in its pure form often shows divergent behavior for function approximation \cite{sutton2018reinforcement, van2018deep}. It has no known convergence guarantees except for some similar algorithms where convergence results have been obtained \cite{fan2020theoretical}.
Two major ideas have been developed to improve (but not guarantee) its convergence: using a target network (Double deep Q-learning, DDQN) and employing experience replay \cite{mnih2015human}.
In Double deep Q-learning, the target network is the exact copy of the actual network but updated less frequently. Freezing the target network prevents the target value from changing faster than the actual Q-value during learning. Intuitively, learning can become unstable and lose convergence if the target changes faster than the actual value.
With experience replay, a memory buffer is introduced. Samples are drawn randomly from this buffer, thus
minimizing the correlation between samples observed in trajectory-based learning and enabling the use of supervised learning techniques that assume sample independence \cite{carvalho2020new}.

\subsection{Neural Tangent Kernel}
\label{sec:NTK}
This section briefly defines the NTK and lists some of its relevant properties. 
\begin{definition}
\label{def:NTK}
\textbf{Neural Tangent Kernel \cite{jacot2018neural}.} Given data $x_i, \; x_j \in X \subseteq \mathbb{R}^n$, the NTK of an
$n$ input $1$ output artificial neural network $f(x, \theta(t)): \; \mathbb{R}^n \rightarrow \mathbb{R}$, parametrized with $\theta(t)$, is 
\begin{equation}
    \Theta(x_i,x_j) =\left( \frac{\partial f(x_i,\theta(t))}{\partial \theta(t)} \frac{\partial f(x_j, \theta(t))}{\partial \theta(t)}^T \right)  \in \mathbb{R}. 
\end{equation}
\end{definition}
\begin{remark}
\textbf{Multiple outputs.} From the NTK perspective, a neural network with $n$ outputs behaves asymptotically (towards the infinite width limit) like $n$ networks with scalar outputs trained independently. I.e., the diagonal elements of the NTK will dominate.
\end{remark}
\begin{remark}
\textbf{Constant kernel.} Although, $\theta(t)$ is changing during training, in the infinite width limit, the NTK converges to an explicit constant kernel. It only depends on the depth, activation function, and parameter initialization variance of an NN. In other words, during training, $\Theta(x_i,x_j)$ is independent of time $t$. 
\end{remark}
\begin{remark}
\textbf{Linear dynamics.} In the infinite width limit, an NN can be well described throughout training by its first-order Taylor expansion (i.e.,~linear dynamics) around its parameters at initialization ($\theta(0)$), assuming a low learning rate \cite{Lee2020Wide}:
\begin{equation}
\label{eq:Taylor}
    f(x,\theta(t)) \approx f(x, \theta(0)) + \frac{\partial}{\partial \theta} f(x, \theta(0))(\theta(t) - \theta(0))
\end{equation}
and $x \in X$.
\end{remark}
\begin{remark}
\textbf{Gradient flow.} 
The NTK describes the evolution of neural networks under gradient descent in function space. Under gradient flow (continuous learning with infinitely low learning rate via gradient descent), the weight update is given as
\begin{equation}
    \frac{d \theta(t)}{dt} = \frac{\partial \mathcal{L}(f(x, \theta(t)))}{\partial \theta(t)},
\end{equation}
with an at least once continuously differentiable (w.r.t $\theta(t)$) arbitrary loss function $\mathcal{L}(f(x, \theta(t)))$ and $x \in X$. 
Then, with the help of the chain-rule and gradient flow, the learning dynamics in Eq.~\eqref{eq:Taylor} becomes
\begin{align}
\frac{d f(x, \theta(t))}{d t}& = \frac{\partial f(x, \theta(t))}{\partial \theta(t)}\frac{d \theta(t)}{dt} =\frac{\partial f(x, \theta(t))}{\partial \theta(t)}\frac{\partial \mathcal{L}(f(x, \theta(t)))}{\partial \theta(t)} \\ \nonumber & = \frac{\partial f(x, \theta(t))}{\partial \theta(t)}\frac{\partial f(x, \theta(t))}{\partial \theta(t)}^T\frac{\partial \mathcal{L}(f(x, \theta(t)))}{\partial f(x, \theta(t))} = \Theta(x,x)\frac{\partial \mathcal{L}(f(x, \theta(t)))}{\partial f(x, \theta(t))}.
\end{align}
\end{remark}

\section{Control-oriented modeling of deep Q-learning}
\label{sec:DQN_modeling}
In this section, deep Q-learning is translated into a dynamical system. 
In light of the properties of the NTK (Section \ref{sec:NTK}), shallow and wide neural networks are used for approximating the Q-function. That is to exploit its constant nature and the opportunity to linearize the learning dynamics.

Assuming $Q_1(t) = Q(s(t), a(t), \theta(t))$ is the actual Q-value and $Q_2(t) = Q(s'(t),a'(t), \theta(t))$ is the next Q-value with $a'(t) = \underset{a \in \mathcal{A}}{\mathrm{argmax}}\, Q(s'(t),a, \theta(t))$ are the system states, Q-learning can be modeled with uncertain continuous, linear time-invariant dynamics. 
Let the NTK of the deep Q-network be evaluated at the current state of the environment $s(t)$ for output $a(t)$. Denote it as $\Theta_1 = \Theta((s(t)|a(t)),(s(t)|a(t)))$. Similarly, for the next state as $\Theta_2 = \Theta((s(t)|a(t)),(s'(t)|a'(t)))$. The role of $\Theta_1$, and $\Theta_2$ is to characterize how $Q_1(t)$ and $Q_2(t)$ will evolve during learning according to Remark 4. In addition, denote the bounded uncertainty block encompassing unmodeled learning behaviour by $\Delta \in \mathbb{R}^{2 \times 2}$, $||\Delta||_\infty < \infty$, where $||\cdot||_\infty$ denotes the infinity norm.  
\begin{theorem}
\textbf{Dynamics of deep Q-learning.}
Deep Q-learning can be modeled as a continuous-time, linear time-invariant system with output multiplicative uncertainty with the help of the NTK as
\begin{equation}
\label{eq:q_learning_uncontrolled_state_space}
\begin{gathered}
    \begin{bmatrix}
\frac{d Q_1(t)}{d t}\\ 
\frac{d Q_2(t)}{d t}\\ 
\end{bmatrix}
=
\left(
\begin{bmatrix}
-\Theta_1 & \gamma \Theta_1 \\ 
-\Theta_2 & \gamma \Theta_2 \\
\end{bmatrix} + \Delta \right)
\begin{bmatrix}
Q_1(t)\\ 
Q_2(t)
\end{bmatrix}
+
\begin{bmatrix}
\Theta_1\\ 
\Theta_2
\end{bmatrix} r(t),  \\
    y(t) = \begin{bmatrix}
    1 & 0 \\ 0 & 1
    \end{bmatrix}\begin{bmatrix}
Q_1(t)\\ 
Q_2(t)
\end{bmatrix}.
\end{gathered}
\end{equation}
\end{theorem}
\begin{proof}
The proof consists of three parts. First, learning dynamics are formulated for fixed state-action values, and the appearance of the NTK is shown. Then, results are cast into a state-space form for selected Q-values. Finally, the necessity of the uncertainty block and its components are discussed. 

\textit{Part 1: Learning dynamics.} In order to describe the learning as a dynamical system, first, the weight update with quadratic loss (Eq.~\eqref{eq:theta_update}) is translated into continuous-time (gradient flow).  Assume $\frac{\theta(t+1)-\theta(t)}{\alpha}$ is the Euler discretization of $\frac{d \theta}{d t}$ \cite{warmuth1997continuous}. If the learning rate $\alpha \rightarrow 0$ the parameter update in continuous time can be written as 
\begin{equation}
\label{eq:Theta_change_0}
    \frac{d \theta(t)}{d t} = (r(t) + \gamma \underset{a \in \mathcal{A}}{\mathrm{max}} Q(s'(t),a, \theta(t)) - Q(s(t), a(t), \theta(t))) \frac{\partial Q(s(t), a(t), \theta(t))}{\partial \theta(t)}^T.
\end{equation}
Based on Eq.~\eqref{eq:Theta_change_0}, the Q-value evolution at state $s(t)$ for action $a(t)$  can be written with the help of the chain-rule as

\begin{align}
\label{eq:Q_change_self} \nonumber &
    \frac{d Q(s(t), a(t), \theta(t))}{d t} = 
    \frac{\partial Q(s(t), a(t), \theta(t))}{\partial \theta(t)} \frac{d \theta(t)}{d t} = \\ & 
     \left(\frac{\partial Q(s(t), a(t), \theta(t))}{\partial \theta(t)} \frac{\partial Q(s(t), a(t), \theta(t))}{\partial \theta(t)}^T \right) 
     (r(t) + \gamma \underset{a \in \mathcal{A}}{\mathrm{max}} Q(s'(t),a, \theta(t)) - Q(s(t), a(t), \theta(t))).
\end{align}
Since the temporal difference (the first term in the right hand side in Eq.~\eqref{eq:Theta_change_0}) is a scalar, it commutes with the vector $\frac{\partial Q(s(t), a(t), \theta(t))}{\partial \theta(t)}^T$. Therefore, in Eq.~\eqref{eq:Q_change_self}, the term $\left( \frac{\partial Q(s(t), a(t), \theta(t))}{\partial \theta(t)} \frac{\partial Q(s(t), a(t), \theta(t))}{\partial \theta(t)}^T \right)  \in \mathbb{R}$ appears. This is the NTK evaluated at $s(t)$ for action $a(t)$: $\Theta((s(t)|a(t)), (s(t)|a(t)))$, see Definition 1. Note that, in this setting, the scalar product is always non-negative as it is the sum of the squared partial derivatives.

Similarly to Eq.~\eqref{eq:Q_change_self}, the Q-value changes at an arbitrary $s_u(t)$, $a_u(t)$ state-action pair can be computed assuming temporal difference update with data tuple $(s(t), a(t), r(t), s'(t))$: 
\begin{align} \nonumber
    & \frac{d Q(s_u(t), a_u(t), \theta(t))}{d t} = 
    \frac{\partial Q(s_u(t), a_u(t), \theta(t))}{\partial \theta} \frac{d \theta(t)}{d t} = \\& \label{eq:Q_change_other} \Theta((s(t)|a(t)),(s_u(t)|a_u(t))) (r(t) + \gamma \underset{a \in \mathcal{A}}{\mathrm{max}} Q(s'(t),a, \theta(t) - Q(s(t), a(t), \theta(t))).
\end{align}
Based on Eq.~\eqref{eq:Q_change_other}, the evolution of the Q-function is only influenced by the NTK.

In the preliminaries, $s'(t))$ was defined as $s(t+1)$. Assuming continuous-time (gradient flow), $s'(t)$ becomes $s(t+\Delta t)$, and $\Delta t \rightarrow 0$. 
 Next, let the arbitrary $s_u(t)$, $a_u(t)$ state-action pair be the next state $s'(t)$ and the best action at that state, denoted by $a'(t)$. Then, the evolution of $Q(s'(t), a'(t)  \theta(t))$ becomes
\begin{align}
\label{eq:Q_change_next} \nonumber
    & \frac{d Q(s'(t), a'(t), \theta(t))}{d t} = 
    \frac{\partial Q(s'(t), a'(t), \theta(t))}{\partial \theta(t)} \frac{d \theta}{d t} = \\& \Theta((s(t)|a(t)),(s'(t)|a'(t))) (r(t) + \gamma \underset{a \in \mathcal{A}}{\mathrm{max}} Q(s'(t),a, \theta(t)) - Q(s(t), a(t), \theta(t))).
\end{align}

\textit{Part 2: State-space.} Using the simplified notations $Q_1(t) = Q(s(t), a(t), \theta(t))$, $Q_2(t) = Q(s'(t),a'(t), \theta(t))$, $\Theta_1 = \Theta((s(t)|a(t)),(s(t)|a(t)))$, and $\Theta_2 = \Theta((s(t)|a(t)),(s'(t)|a'(t)))$ the two first-order inhomogeneous linear ODEs (Eq.~\eqref{eq:Q_change_self} and Eq.~\eqref{eq:Q_change_next}) can be organized into state-space form with the system states being $[Q_1(t), \; Q_2(t)]^T$ and assuming the reward $r(t)$ is an exogenous signal. Then, the nominal system becomes
\begin{equation}
\label{eq:q_learning_uncontrolled_state_space_nominal}
    \begin{bmatrix}
\frac{d Q_1(t)}{d t}\\ 
\frac{d Q_2(t)}{d t}\\ 
\end{bmatrix}
=
\begin{bmatrix}
-\Theta_1 & \gamma \Theta_1 \\ 
-\Theta_2 & \gamma \Theta_2 \\
\end{bmatrix}
\begin{bmatrix}
Q_1(t)\\ 
Q_2(t)
\end{bmatrix}
+
\begin{bmatrix}
\Theta_1\\ 
\Theta_2
\end{bmatrix} r(t).
\end{equation}
Learning dynamics are characterized by the NTKs $\Theta_1$ and $\Theta_2$ in the coefficient matrices. 

\textit{Part 3: Uncertainties.}
Despite its simple form, this system is inherently uncertain. This uncertainty stems from a single source but manifests in three forms that are specific for reinforcement learning. In contrast to a supervised learning setting, where data is static, in reinforcement learning, data is obtained sequentially as the agent explores the environment.
\begin{itemize}
    \item \textit{Changing environment states.} The system states $Q_1(t)$ and $Q_2(t)$ have unmodeled underlying dynamics as they always correspond to different $s(t)$, $s'(t)$ environment states and actions (recap: $Q: \mathcal{S} \times \mathcal{A} \rightarrow \mathbb{R}$). On the other hand, if slow learning rate is assumed, and the Q-function is smooth, the deviation from the modeled Q-values is bounded. This deviation can be included into the modeling framework as an output multiplicative uncertainty $\Delta_Q = \begin{bmatrix} \Delta_{Q_1} & 0 \\ 0 & \Delta_{Q_2} \end{bmatrix}$, overbounding the temporal variation of the states. It is assumed this uncertainty is proportional to the magnitude of the Q-values. 
    \item \textit{Parametric uncertainty in the NTK.} Dynamically changing environment states cause parametric uncertainty through the NTK. Although the NTK seldom changes during training for wide neural networks (Remark 2), it is only true if the data (where the NTK is evaluated) is static. This is not the case in reinforcement learning: it has to be evaluated for different $(s(t)|a(t))$, $(s'(t)|a'(t))$ pairs in every step. 
    Since both $Q_1(t)$ and $Q_2(t)$ are known, the NTK can be computed in every step. On the other hand, that would lead to a parameter-varying system. On the other hand, since the actual NTK values are only influenced by data, upon initialization of the neural network, it can be evaluated at several environment state pairs to estimate its bounds offline. Parametric uncertainties can form nonconvex regions which can only be handled via robust control by overbounding there regions. To this end, the parametric uncertainty is pulled out from the plant and overbounded by convex and unstructured uncertainty structure. In particular, it is captured with an output multiplicative uncertainty, see Figure \ref{fig:unc_nyquist} This technique is discussed comprehensively in  \cite{zhou1998essentials}. Finally, the following output multiplicative uncertainty structure is defined, 
    $\Delta_\Theta = \begin{bmatrix} \Delta_{\vartheta} & 0.1\Delta_{\vartheta} \\ 0.1\Delta_{\vartheta} & \Delta_{\vartheta} \end{bmatrix}$.
    \begin{figure}[htb!]
    \centering
    \includegraphics[width=0.2\linewidth]{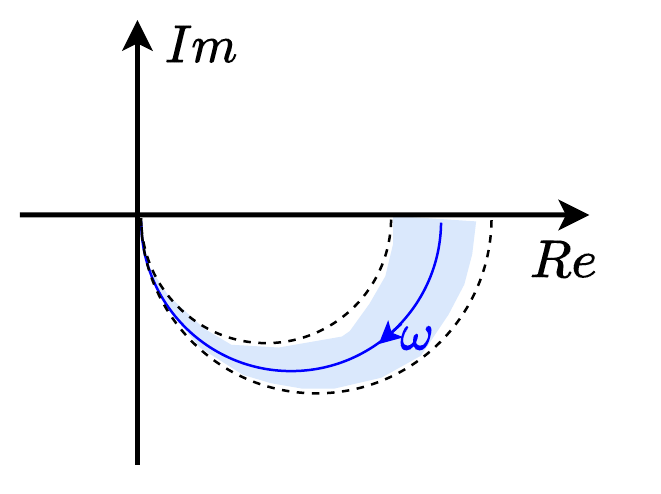}
    \caption{The parametric uncertainty makes the frequency response of the system vary within nonconvex bounds, depicted with blue regions in this Nyquist diagram. The output  multiplicative uncertainty overbounds this variation.}
    \label{fig:unc_nyquist}
    \end{figure}
    \item \textit{Exploration.} Exploration in deep Q-learning means taking an action that do not correspond to the highest Q-value at $s(t)$. Thus, $Q_2(t)$ may not be $\underset{a \in \mathcal{A}}{\mathrm{max}} Q(s'(t),a, \theta(t))$, rather randomly selected $Q$ value. It can be lumped into the previously introduced output multiplicative uncertainty terms as $\Delta_{Exp} = \begin{bmatrix} 0 & 0 \\ 0 & \Delta_{\hat Q_2} \end{bmatrix}$. 
\end{itemize}
Note that none of the uncertainty blocks are time-dependent but bounded. That is because this model proposed to overbound all possible uncertainties in a robust way. Then, all uncertainty components are combined into a single uncertainty block
\begin{equation}
\Delta = \Delta_{Q} + \Delta_{\Theta} + \Delta_{Exp}. 
\end{equation}
Finally, assuming the output of the single input, multiple output system is $Q_1(t)$, and $Q_2(t)$, the uncertain LTI model of deep Q-learning is
\begin{equation}
\begin{gathered}
\begin{bmatrix}
\frac{d Q_1(t)}{d t}\\ 
\frac{d Q_2(t)}{d t}\\ 
\end{bmatrix}
=
\left( \begin{bmatrix}
-\Theta_1 & \gamma \Theta_1 \\ 
-\Theta_2 & \gamma \Theta_2 \\
\end{bmatrix} + \Delta \right) \begin{bmatrix}
Q_1(t)\\ 
Q_2(t)
\end{bmatrix}
+
\begin{bmatrix}
\Theta_1\\ 
\Theta_2
\end{bmatrix} r(t),  \\
    y(t) = \begin{bmatrix}
    1 & 0 \\ 0 & 1
    \end{bmatrix}\begin{bmatrix}
Q_1(t)\\ 
Q_2(t)
\end{bmatrix}.
\end{gathered}
\end{equation}
\end{proof}
Next, through a series of remarks, some properties of this system are outlined. 

\begin{remark} \textbf{Uncertainty structure.} It would be possible to select different error structures for the unmodeled dynamics. For example, an input multiplicative uncertainty would make more sense for the exploration uncertainty. However, for simplicity, it is assumed it is an output multiplicative uncertainty. Alternatively, it could be handled as parametric uncertainty directly via $\mu$-synthesis \cite{stein1991beyond}.
\end{remark}
\begin{conjuncture} \textbf{Nominal stability.}
The stability of the linearized deep-learning dynamics is easy to check. The nominal linear system is stable if the real parts of the $2\times 2$ system matrix's eigenvalues are negative, i.e,
\begin{equation}
eig \left( \begin{bmatrix}
-\Theta_1 & \gamma \Theta_1 \\ 
-\Theta_2 & \gamma \Theta_2 \\
\end{bmatrix}
 \right) = [\lambda_1, \lambda_2], \; \text{if } Re(\lambda_1), \; Re(\lambda_2) < 0 \;  \text{then asymptotically stable}.  
\end{equation}
The state matrix above has one zero eigenvalue, while the other eigenvalue is $\gamma \Theta_2 - \Theta_1$. Thus, the system describing Q-learning is locally asymptotically  stable if $\Theta_1 > \gamma \Theta_2$. The magnitude of the NTK is related to the rate of change of the function approximator during learning. Intuitively, if $Q_2(t)$ (the target) is changing faster (dictated by $\gamma\Theta_2$) than the actual value $Q_1(t)$ (dictated by $\Theta_1$), learning will not converge. This result supports the divergence claim of standard deep Q-learning \cite{van2018deep}.
\end{conjuncture}
\begin{remark}\textbf{Relation to Double deep Q-learning.}
A common remedy for the divergent behavior  of Q-learning is the target network \cite{mnih2015human}. I.e., $Q_2(t)$ is computed from an independent but identical neural network which is less frequently updated. In the control-oriented modeling framework, this would mean a piecewise static $Q_2(t)$, with $\Theta_2 = 0$. Since $\Theta_1 \geq 0$, the state space representation of Double deep Q-learning would be asymptotically stable for all $\Theta_1$. This remark highlights the efficiency of DDQN from an alternative perspective. 
\end{remark}
\begin{remark}{\textbf{Boundedness of the parametric uncertainty.}}
In reinforcement learning, the NTK changes due to the dynamically changing data. Therefore, the bounds of the NTK can be evaluated by computing $\Theta_1$ and $\Theta_2$ for a set of environment states in a grid-based fashion, offline, assuming the environment states are bounded too. Figure \ref{fig:NTK_range}
depicts a slightly different approach: actual state transitions are taken from one of the simulation case studies (Section \ref{sec:cartpole}). This significantly reduces the domain where the NTK is evaluated. In addition, it highlights another important property: $\Theta_1$ and $\Theta_2$ are correlated, since both values are computed with the same kernel. Exploiting this correlation can greatly reduce the range of the parametric uncertainty. 
\begin{figure}[htb!]
\centering
\includegraphics[width=0.5\linewidth]{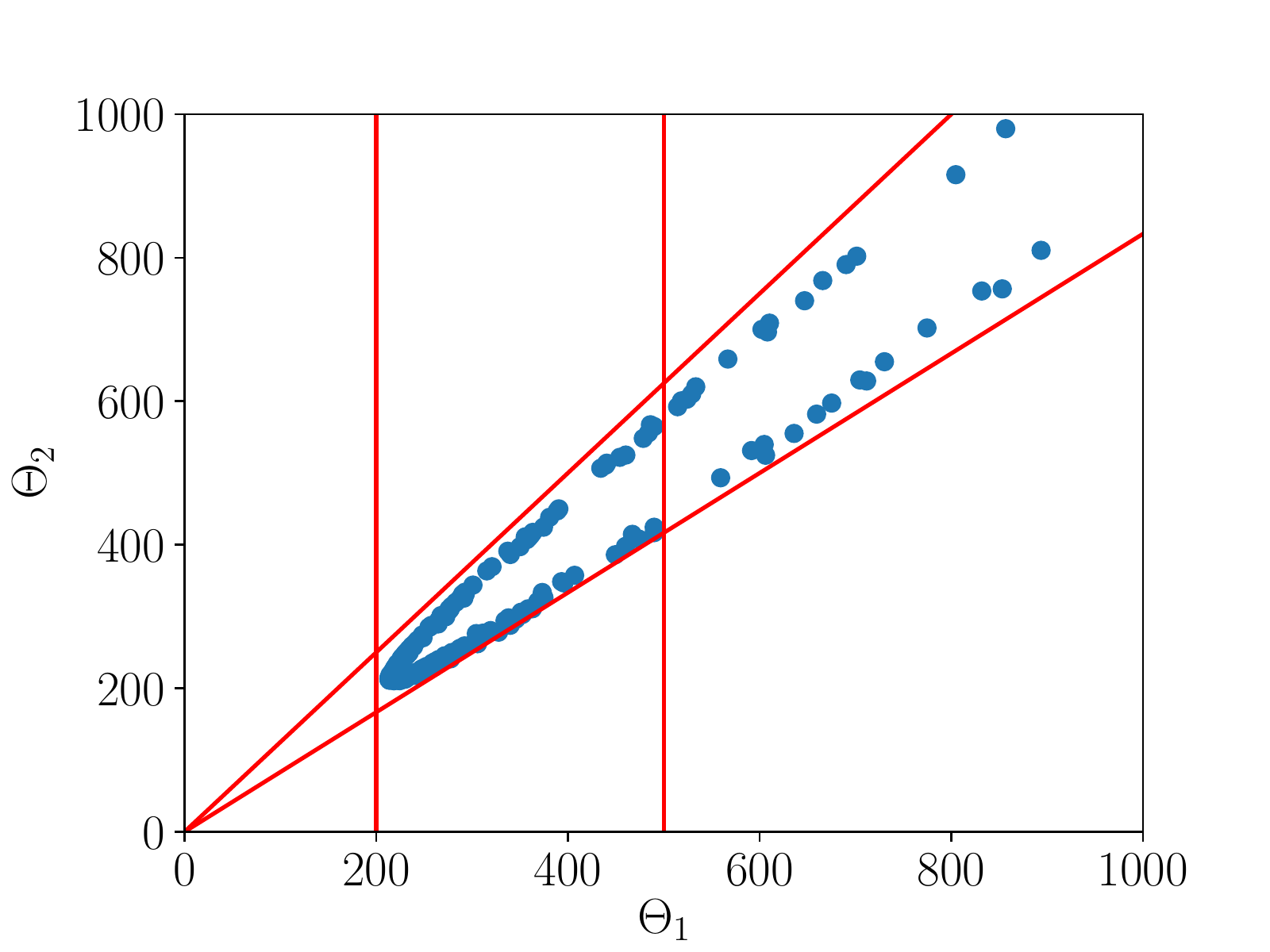}
\caption{Evaluations of the NTK during uncontrolled learning in the Cartpole environment. The NTKs are bounded by the red lines: $200 < \Theta_1 < 500$, $0.8\Theta_1 < \Theta_2 < 1.2\Theta_1$. Note: in the controlled cases, $\Theta_1 > 500$ will be less frequent and only happen when exploring.}
\label{fig:NTK_range}
\end{figure}
\end{remark}
\begin{remark} \textbf{Frequency of the learning.}
Using Fast Fourier Transform (FFT), the frequency content of the agent's input and output signals can be analyzed. Results suggest that these signals are slowly changing. We hypothesize that is due to the smoothness of the Q function (plus low learning rate) and the rewarding scheme; thus learning is in the low-frequency domain (regardless of the environment and control strategy). Figure \ref{fig:FFT} depicts the FFT of $r(t)$, $Q_1(t)$, and $Q_2(t)$ for a controlled Cartpole scenario. 
\begin{figure}[htb!]
\centering
\begin{subfigure}{.25\textwidth}
  \centering
  \includegraphics[width=1\linewidth]{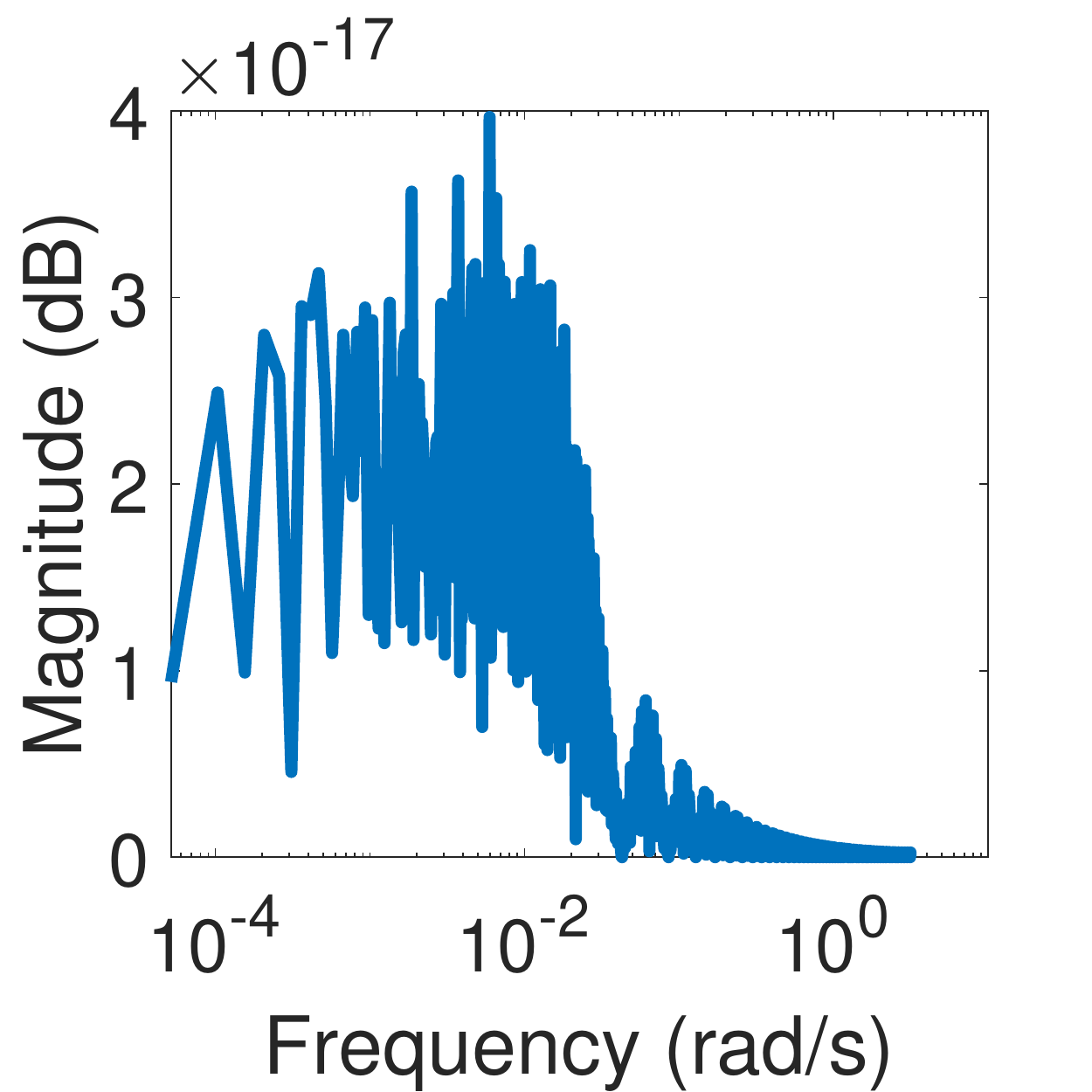}
  \caption{FFT of $r(t)$}
\end{subfigure}%
\begin{subfigure}{.25\textwidth}
  \centering
  \includegraphics[width=1\linewidth]{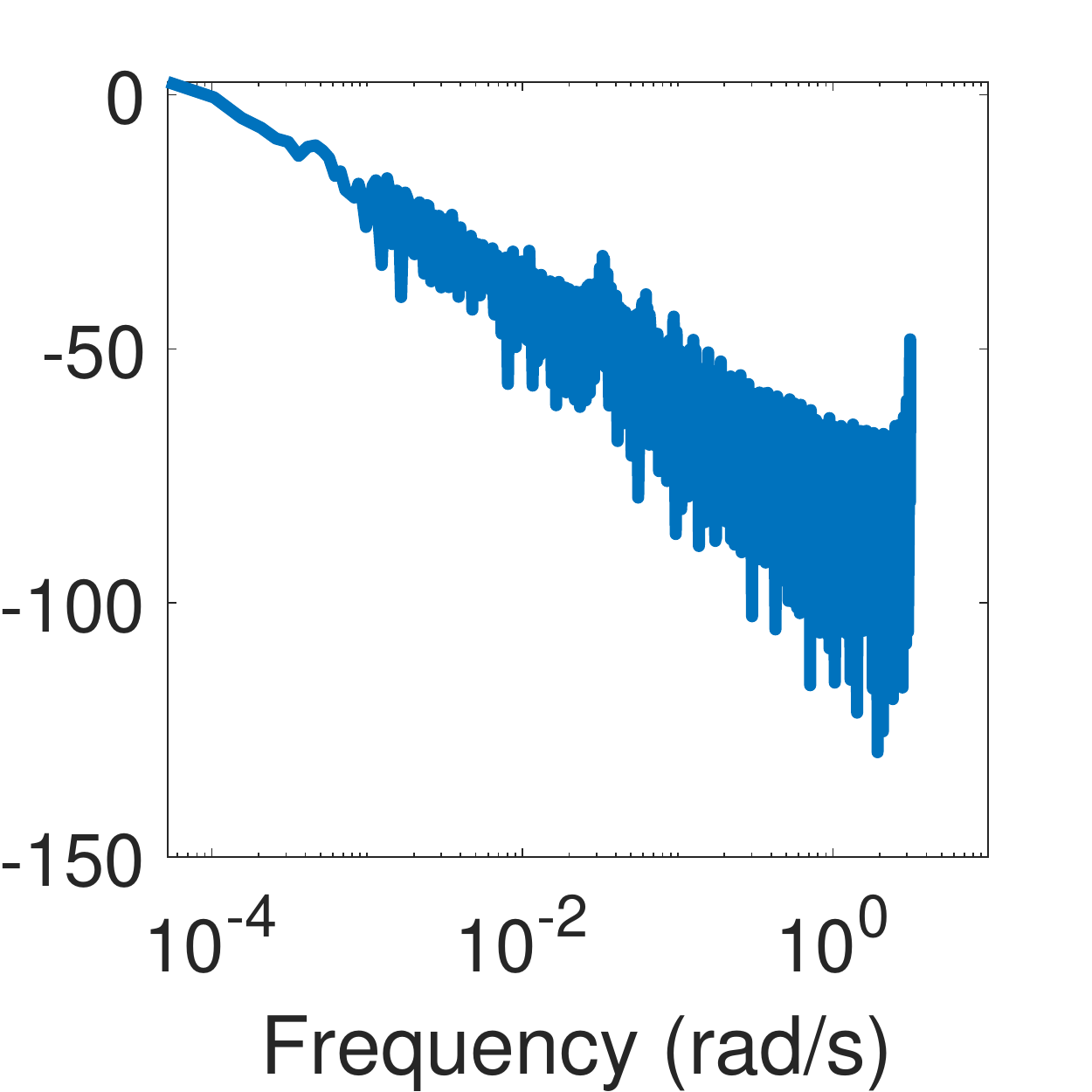}
  \caption{FFT of $Q_1(t)$}
\end{subfigure}
\begin{subfigure}{.25\textwidth}
  \centering
  \includegraphics[width=1\linewidth]{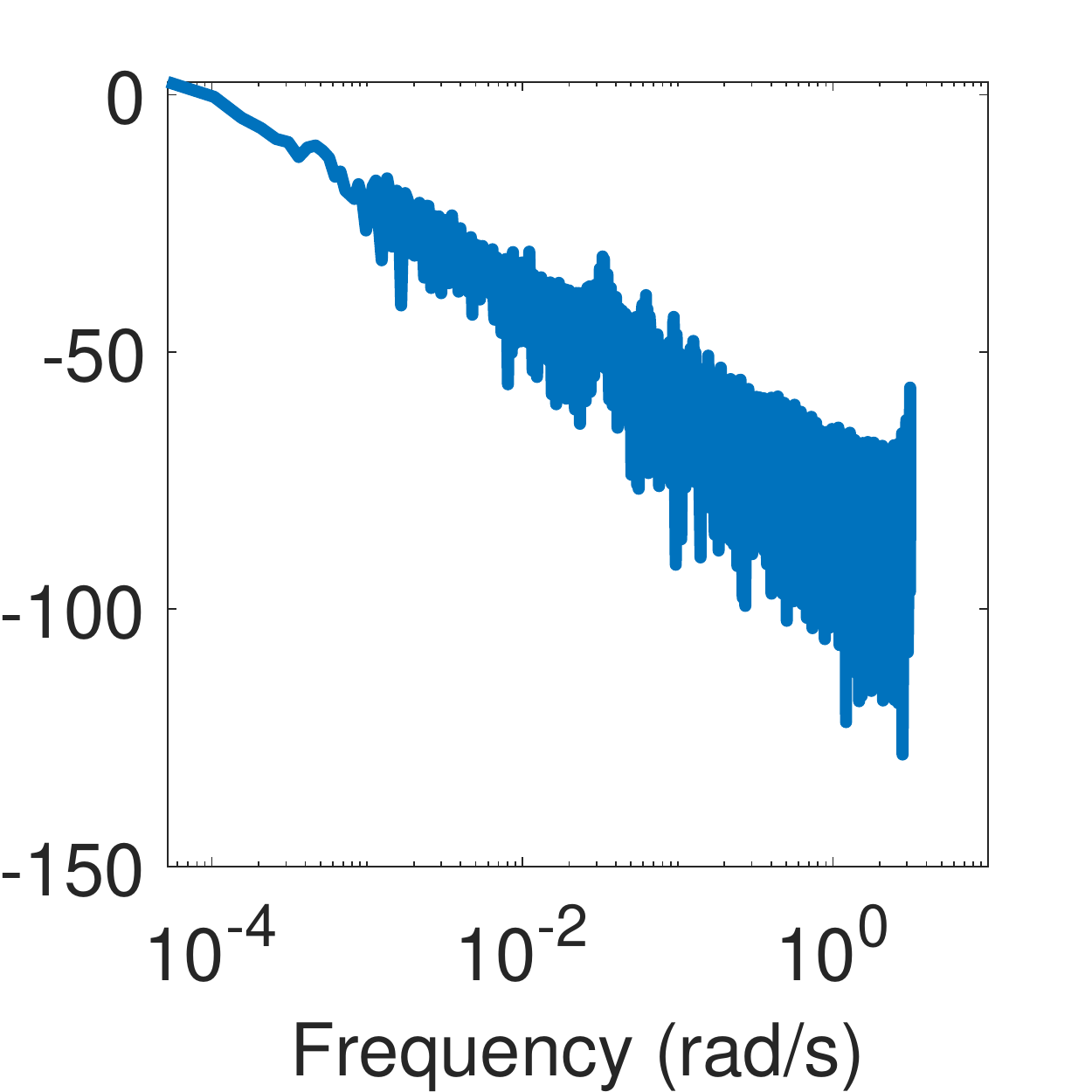}
  \caption{FFT of $Q_2(t)$}
\end{subfigure}
\caption{FFT of the agent's input and output signals in a controlled Cartpole environment.}
\label{fig:FFT}
\end{figure}
\end{remark}
In the next section, stabilizing controllers are formulated based on the linearized learning dynamics (Eq.~\eqref{eq:q_learning_uncontrolled_state_space}).

\section{Explicitly controlled deep Q-learning}
Learning is supported with a cascade control layout that prevents divergent learning behavior for any state-action combination. To this end, the common agent-environment interaction (Figure \ref{fig:RL_setting}) is augmented with an additional feedback controller $K$, as depicted in Figure \ref{fig:dqn_cascade}. In the sequel, the effect of the control on learning is dissected via injecting it into controlled loss functions. In particular,  three different controllers are synthesized and compared: gain scheduling $\mathcal{H}_2$ output-feedback control, dynamic $\mathcal{H}_\infty$, and fixed-structure robust $\mathcal{H}_\infty$. 
\begin{figure}[htb!]
\centering
\includegraphics[width=0.25\linewidth]{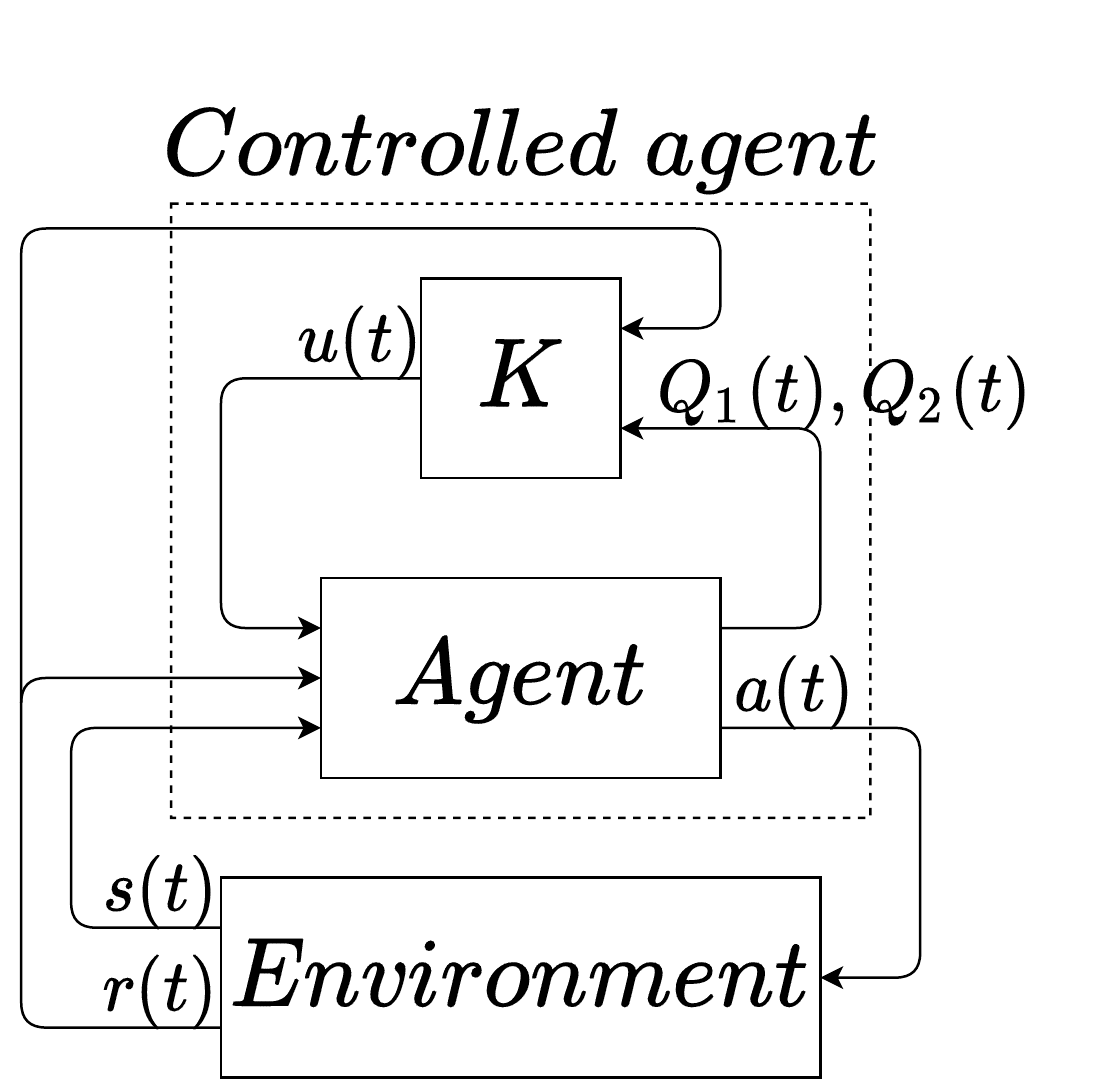}
\caption{Deep Q-learning cascade feedback control}
\label{fig:dqn_cascade}
\end{figure}

\begin{remark} \textbf{Random experience replay.}
Opposed to several Q-learning variants, this approach do not use random experience replay. Instead, the sequential nature of the data is exploited when computing the tracking error. However, it is possible to log episodic trajectories and replay them to the agent to help learning. This method is advantageous in sparse reward environments.
\end{remark}

\subsection{Static output-feedback $\mathcal{H}_2$ controller design}
\label{sec:LQ_control}
The system (Eq.~\eqref{eq:q_learning_uncontrolled_state_space}) can be stabilized via an output-feedback controller. First, the control input $u(t)$ (an additional stabilizing reward) is introduced into the state-space representation of learning as 
\begin{equation}
\label{eq:controlled_Q}
    \begin{bmatrix}
\frac{d Q_1(t)}{d t}\\ 
\frac{d Q_2(t)}{d t}\\ 
\end{bmatrix}
=
\begin{bmatrix}
-\Theta_1 & \gamma \Theta_1 \\ 
-\Theta_2 & \gamma \Theta_2 \\
\end{bmatrix}
\begin{bmatrix}
Q_1(t)\\ 
Q_2(t)
\end{bmatrix}
+
\begin{bmatrix}
\Theta_1\\ 
\Theta_2
\end{bmatrix} r(t) 
+
\Delta 
\begin{bmatrix}
Q_1(t)\\ 
Q_2(t)
\end{bmatrix}
+
\begin{bmatrix}
\Theta_1\\ 
0
\end{bmatrix} u(t).
\end{equation}
As long as $\Theta_1 \neq 0$, the system is controllable \cite{skogestad2007multivariable}.  For control design purposes, create an augmented model via appending a tracking error state $e(t)$ to the state-space. Minimizing $e(t)$ forces the controlled $Q_1(t)$ values to asymptotically converge to the target $r(t) + \gamma \hat Q_2(t)$ (if it is frozen).  The error $e(t)$, $e(0) = 0$ and the extra state $x_e(t)$ are written as
\begin{equation}
    e(t) = \dot x_e(t) = r(t) + \gamma \hat Q_2(t) - Q_1(t),
\end{equation}
and
\begin{equation}
\label{eq:xi}
    x_e(t) = \int_0^{t} r(\tau) + \gamma \hat Q_2(\tau) - Q_1(t) d \tau,
\end{equation}
 respectively. Then, the augmented state-space model  from concatenating Eq.~\eqref{eq:controlled_Q} and Eq.~\eqref{eq:xi} becomes
\begin{equation}
\label{eq_ext_statespace}
    \begin{bmatrix}
\frac{d Q_1(t)}{d t}\\ 
\frac{d Q_2(t)}{d t} \\
\dot x_e(t)
\end{bmatrix}
=
\begin{bmatrix}
-\Theta_1 & \gamma \Theta_1 & 0\\ 
-\Theta_2 & \gamma \Theta_2 & 0\\
-1 & 0 & 0
\end{bmatrix} 
\begin{bmatrix}
Q_1(t)\\ 
Q_2(t)\\
x_e(t)
\end{bmatrix}
+
\begin{bmatrix}
\Theta_1\\ 
\Theta_2\\
1
\end{bmatrix}  r(t) 
+
\Delta_{a}  
\begin{bmatrix}
Q_1(t)\\ 
Q_2(t)\\
0
\end{bmatrix}
+
\begin{bmatrix}
\Theta_1\\ 
0 \\
0
\end{bmatrix}  u(t)
+ 
\begin{bmatrix}
0 \\ 
0 \\
\gamma
\end{bmatrix} \hat Q_2(t),
\end{equation}
where $x_{a}(t) = [Q_1(t),\; Q_2(t),\; x_e(t)]^T$  represent the augmented states, $\Delta_{a} = diag \{ \Delta, \; 0\} \in \mathbb{R}^{3 \times 3}$ is the error augmented error structure, and $\hat Q_2(t) = Q_2(t)$. Note that, no additive error is assumed for the augmented state.

\begin{remark} \textbf{Exogenous $\hat Q_2(t)$.}
If $Q_2(t)$ were included directly, the system would be rank deficient (i.e.,~rows 1 and 3 are multiplies of each other (with factor $\Theta_1$)), yielding a zero eigenvalue. Thus, the system would not be controllable. Consequently, stabilizability is not met either. 
If $\hat Q_2(t)$ is an external signal, it can be chosen freely (e.g.,~as $\hat Q_2(t) = Q_2(t)$) without affecting the dynamical properties of the closed-loop system. 
\end{remark}

An optimal, gain scheduling output-feedback $\mathcal{H}_2$ controller can be realized assuming some properties of the uncertainty block and the external signals. $\mathcal{H}_2$ controller cannot handle the uncertainty block $\Delta_a$ explicitly. Thus, it is handled in two parts: the parametric uncertainty is computed explicitly in every step, while the rest of the uncertainties are neglected. The scheduling parameter $\rho$ captures the variation of $\Theta$ from a nominal one in an affine way. The bounds of $\rho$ stem from Remark 7.  Additionally, it is assumed that there exist a stabilizing controller for every $\rho \neq 0$.
Next, the coefficient matrices of the augmented model in Eq.~\eqref{eq_ext_statespace} are encompassed into $P_a$, and write
\begin{equation}
\left[
     \begin{array}{c}
\frac{d Q_1(t)}{d t}\\ 
\frac{d Q_2(t)}{d t} \\
\dot x_e(t) \\ \hline
z_{p,Q_1}(t) \\
z_{p,Q_2}(t) \\
z_{p,x_e}(t) \\
z_u(t)
\end{array}
\right]
=
\underbrace{
\left[
     \begin{array}{c c c | c c c}
     -\Theta_1(\rho) & \gamma \Theta_1(\rho) & 0 & \Theta_1(\rho) & 0 & \Theta_1(\rho) \\
     -\Theta_2(\rho) & \gamma \Theta_2(\rho) & 0 & \Theta_2(\rho) & 0 & 0 \\
     -1 & 0 & 0 & 1 & \gamma & 0  \\ \hline
     w_{x,Q_1}^{\frac{1}{2}} & 0 & 0 & 0 & 0 & 0 \\
     0 & w_{x,Q_2(t)}^{\frac{1}{2}} & 0 & 0 & 0 & 0 \\
     0 & 0 & w_{x,x_e}^{\frac{1}{2}} & 0 & 0 & 0 \\
     0 & 0 & 0 & 0 & 0 & W_c^{\frac{1}{2}} \end{array}
\right]
}_{P_a}
\left[
     \begin{array}{c}
Q_1(t)\\ 
Q_2(t) \\
x_e(t) \\ \hline
r(t) \\
\hat Q_2(t) \\
u(t)
\end{array}
\right]
\end{equation}
with $z(t) = [z_{p,Q_1}(t),\; z_{p,Q_2}(t),\; z_{p,x_e}(t),\; z_{u}]^T$, and $W_x = diag\{ w_{p,Q_1},\; w_{p,Q_2},\; w_{p,x_e} \}$. Instead of handling the parametric uncertainty in a linear parameter varying (LPV) way, a locally optimal controller is constructed every step via solving the controller design problem repeatedly.
The goal is finding a stabilizing optimal controller $u(t) = - K(\rho)x_a(t)$ that minimizes the lower linear fractional transformation (LFT) $||\mathcal{F}_l(P_a, K(\rho))||_2$ for every $\rho$ in a gain-scheduled manner as
\begin{equation}
    \underset{K(\rho)}{min} ||\mathcal{F}_l(P_a, K(\rho))||_2 = \underset{K(\rho)}{min} \sqrt{\int_0^\infty z^T(t) z(t) dt},
\end{equation}
which turns into the following quadratic optimization problem \cite{skogestad2007multivariable}: 
\begin{equation}
     \underset{K(\rho)}{min} J(x_a(t), u(t)) = \frac{1}{2} \int_0^{\infty} x_a^T(t)^T W_x x_a(t) + u(t)^T W_c u(t) dt. 
\end{equation}
The solution to the above optimization can be given in a closed-form, yielding the he Control Algebraic Ricatti Equation, \cite{kwakernaak1972linear}.

$W_x \succeq 0$, and $W_c > 0$ are positive (semi-)definite diagonal weighting matrices, serving as tuning parameters for the controller. $W_x$ penalizes the performance, including the tracking error, and $W_c$ penalizes the control input. Assigning high diagonal elements to $W_x$ emphasizes on tracking: in this case the error shall be minimized, i.e.,~$w_{p,x_e}$ shall be high compared to the other diagonal elements and $W_c$. The reason behind this is that the Q-values shall not be minimized. The weight for the control input $W_c$ can be kept low too (cheap control \cite{hespanha2018linear}), because the extra reward in the form of $u(t)$ does not have a physical meaning, does not result in excess energy consumption. On the other hand, $u(t)$ acts as an arbitrary reward that  will distort the learning dynamics. 
Denoting the constant elements of the controller as $K = [k_1,\; k_2,\;k_3]^T$, the closed-loop system can be written as
\begin{align} 
    \begin{bmatrix}
\frac{d Q_1(t)}{d t}\\ 
\frac{d Q_2(t)}{d t}\\
\dot x_e(t)
\end{bmatrix}
&= 
\begin{bmatrix}
-\Theta_1(\rho)(1+k_1) & (\gamma-k_2) \Theta_1(\rho) & \Theta_1(\rho) k_3\\ 
-\Theta_2(\rho) & \gamma \Theta_2(\rho) & 0\\
-1 & 0 & 0
\end{bmatrix}
\begin{bmatrix}
Q_1(t)\\ 
Q_2(t)\\
x_e(t)
\end{bmatrix}  +
\begin{bmatrix}
\Theta_1(\rho)\\ 
\Theta_2(\rho)\\
1
\end{bmatrix} r(t) 
+ 
\begin{bmatrix}
0 \\ 
0 \\
\gamma
\end{bmatrix} \hat Q_2(t).
\end{align}
The control input only affects $\frac{d Q_1(t)}{d t}$ directly. Assuming $\hat Q_2(t) = Q_2(t)$, and computing the actual $\Theta_1$ corresponding to $\Theta_1(\rho)$ the Q-value change at $s(t)$, $a(t)$ can be written as 
\begin{equation}
 \label{eq:Q_change_LQ}
 \frac{d Q_1(t)}{d t} =  \Theta_1 \left(r(t) + (\gamma - k_2) Q_2(t) - (1 +  k_1 )Q_1(t) +  k_3 x_e(t) \right)
\end{equation}
The controller gains can be placed inside the parenthesis since the control input acts through $\Theta_1$. $k_1$ and $k_2$ affect the stability of Q-learning, while $k_3$ influences the tracking error.  

Next,  the controlled loss function is calculated based on Eq.~\eqref{eq:Q_change_LQ}. Recall the chain-rule $\frac{d\theta}{d t}= \frac{d \theta}{d Q_1(t)} \frac{dQ_1(t)}{dt}$ and use the definition of the NTK (with detailed notations) $\Theta_1 = \frac{\partial Q(s(t), a(t), \theta)}{\partial \theta} \frac{\partial Q(s(t), a(t), \theta)}{\partial \theta}^T$ to achieve  \begin{align}
 \label{eq:theta_change_LQ}
 \frac{d\theta}{d t} = & \left(r(t) + (\gamma - k_2) \underset{a}{\mathrm{max}} Q(s(t+1),a, \theta) - (1 +  k_1 )Q(s(t), a(t), \theta) +   k_3 x_e(t)
 \right) \frac{\partial Q(s(t), a(t), \theta)}{\partial \theta}^T.
\end{align}
Then, the controlled loss is obtained by integrating Eq.~\eqref{eq:theta_change_LQ} with respect to $\theta$. In order to obtain a similar form to Eq.~\eqref{eq:loss_0}, only $\theta$ dependency is assumed for $Q(s(t), a(t), \theta)$ when performing the integration. This assumption is based on the following arguments.
\begin{itemize}
    \item When evaluating the weight evolution (Eq.~\eqref{eq:loss_0}), the more common direct method is used over a residual gradient method \cite{baird1995residual}, thus the $\theta$ dependency of the temporal difference target \\$r(t) + \gamma \underset{a}{\mathrm{max}} Q(s(t+1),a, \theta)$ is not considered.
    \item The controller gains $k_1$, $k_2$, and $k_3$ are constants. 
    \item The terms in the integral $x_e(t)$ (Eq.~\eqref{eq:xi}) depend only on past values of $\theta$ (implicitly). Therefore, this term can be considered constant when integrating with respect to $\theta$. 
\end{itemize}
Integrating Eq.~\eqref{eq:theta_change_LQ} with respect to $\theta$ and considering the above assumptions the controlled loss becomes
\begin{align} \nonumber
 \mathcal L_{\mathcal{H}_2} = \int_0^\infty &\left(r(t) + (\gamma - k_2) \underset{a}{\mathrm{max}} Q(s(t+1),a, \theta) - (1 +  k_1 )Q(s(t), a(t), \theta) + k_3 x_e(t)
 \right) \frac{\partial Q(s(t), a(t), \theta)}{\partial \theta}^T d\theta  \\  & \; = \frac{1}{2(1 + k_1)}\left(r(t) + (\gamma - k_2) \underset{a}{\mathrm{max}} Q(s(t+1),a, \theta) - (1 +  k_1 )Q(s(t), a(t), \theta) + k_3 x_e(t)
 \right)^2, 
 \label{eq:loss_LQ}
\end{align}
Eq.~\eqref{eq:loss_LQ} is the loss function for the controlled agent. The terms in the loss are weighted by the controller's parameters, helping convergence at the cost of biasing the true Q-values. Note that the controller is designed for the nominal plant without considering all the uncertainties in $\Delta$. Although the controller is conservative, dynamic Q-value stabilization is only guaranteed in a local sense. With this approach, uncertainties in $Q_1(t)$ and $Q_2(t)$ cannot be handled explicitly. On the other hand, the controller is inherently robust up to a multiplicative uncertainty of $0.5$ \cite{marcos2004development}.
The parametric uncertainty is handled in a gain scheduling way. The optimal controller can be recomputed every episode via evaluating the NTK  repeatedly. This is computationally intensive and bears the risk that a parameter combination occurs that cannot be stabilized, rendering the learning divergent. In the sequel, the convergence of deep Q-learning is aided by robust control: instead of considering fixed-parameter combinations, the variations in the parameters and the states are explicitly included in the controller design. 
\subsection{$\mathcal{H}_\infty$ controller design}
\label{sec:Hinf_control}
In this section  two types of robust $H_\infty$ controllers are proposed. First, the controller design procedure is outlined for a generic robust dynamical controller where a linear time-invariant system computes the control input. Second, the structure of the controller is fixed and constant gains are utilized akin to the $\mathcal{H}_2$ output-feedback controller.

Although the parametric uncertainty could be explicitly computed, it is computationally inefficient and would lead to a parameter-varying control as demonstrated for the gain scheduling $\mathcal{H}_2$ case. This inefficiency motivates the formulation of a robust controller: it can be synthesized before learning, and it will be stabilizing during learning for all combinations of states and parameters (see Figure \ref{fig:NTK_range}). 
The $\mathcal{H}_\infty$ design framework is capable of handling every uncertainty in $\Delta$ in a robust way. Furthermore, in the $\mathcal{H}_\infty$ controller design procedure the system's response is shaped via dynamically weighting the inputs and outputs of the system. Therefore, the low-frequency nature of the controlled learning agent can be exploited too.  

The aim is controlling the nominal system $\mathcal{P}$, encompassing Eq.~\eqref{eq:controlled_Q}, disturbed by noise through the $\Delta$ block. The controller $K$ has three inputs: the two noisy system states fed back and the reference signal, identical to the tracking error of the $\mathcal{H}_2$ controlled case: $e(t) = r(t) + \gamma \hat Q_2(t) - Q_1(t)$.  In the $\mathcal{H}_\infty$ design, performance is enforced through the tuneable weights that give the desired shape to the singular values of the open-loop response  Selecting the type of the dynamic weight and their parameters is problem dependent and based on heuristics, but these heuristics are well studied in e.g., \cite{skogestad2007multivariable, wu1995control, marcos2004development}.
\begin{itemize}
    \item $W_u$ penalizes the control input and $W_p$ penalizes the error $e$. As discussed before, learning is done in the low-frequency range. Therefore, good tracking performance (large $W_p$) is desired at low frequencies. Although the control input has no physical interpretation, it should be dynamically weighted too. At higher frequencies, tracking shall be penalized more in order to reduce the singular values.
    \begin{equation}
        W_p(i\omega) = \frac{0.001}{0.1 i\omega + 1},
    \end{equation}
    \begin{equation}
        W_u(i\omega) = \frac{0.01i\omega}{(0.1 i\omega + 1)(0.001 i \omega + 1)}.
    \end{equation}
    Bode diagrams of $W_u$ and $W_p$ are shown in Figure \ref{fig:weight_bode}. Note that it turns out that these weights are universal regardless of the RL environment. 
    \begin{figure}[htb!]
    \centering
    \includegraphics[width=0.5\linewidth]{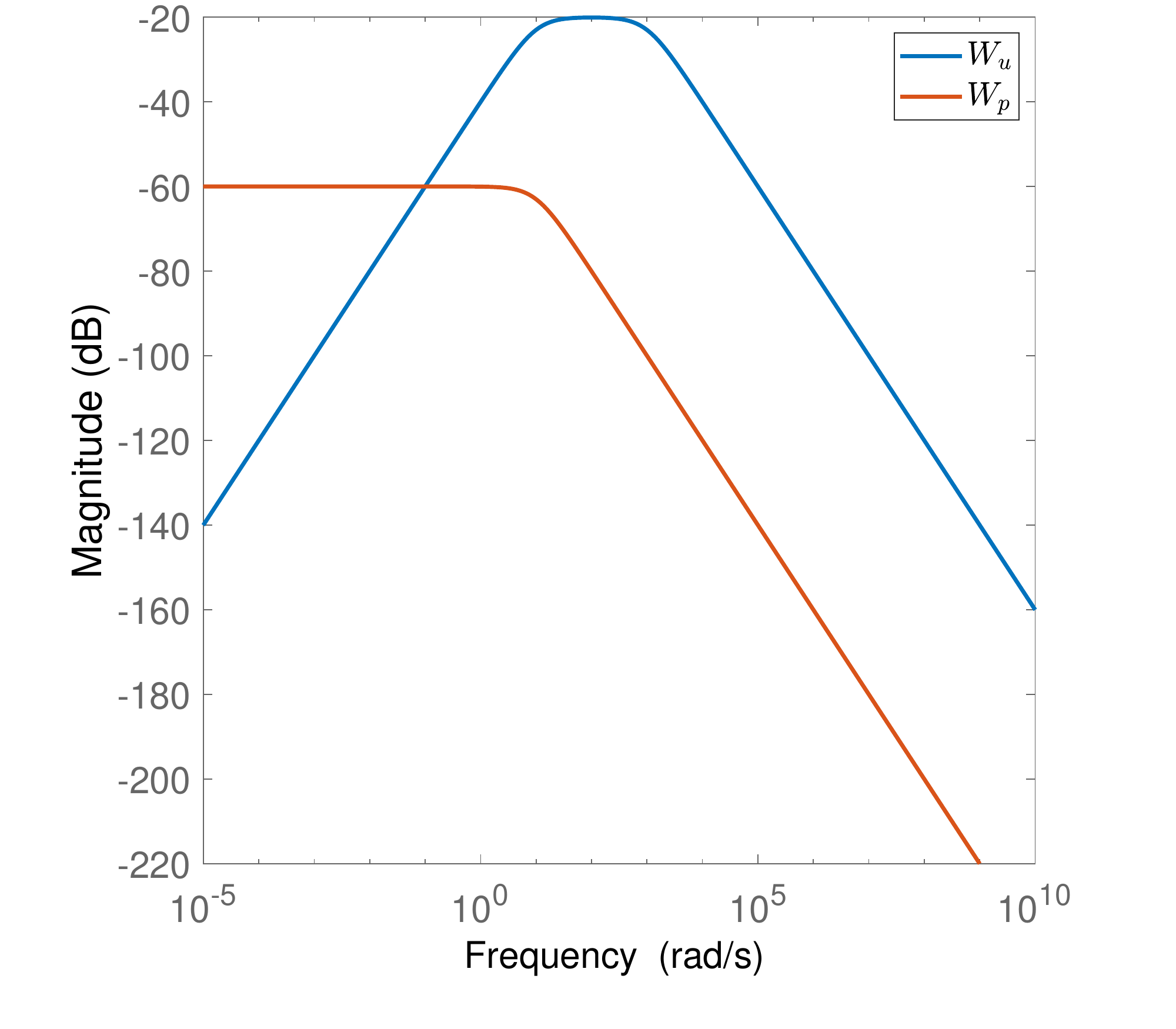}
    \caption{Bode magnitude diagrams of the frequency-dependent tuning weights.}
    \label{fig:weight_bode}
    \end{figure}
    \item $W_\Delta\in \mathbb{R}^{2\times 2}$ shapes the uncertainty. It is considered constant (with varying magnitude from environment to environment) but generally, it can be stated that $\Delta$ is constant at low frequencies, where learning is meaningful. Its magnitude has a peak at an extremely high frequency ($10^{12} rad/s$), which is unimportant for the learning.
    \item The purpose of $W_{r}$ and $W_{\hat Q}$ are to normalize and inject reference signal related dynamism to the reference signals. Here, they are considered frequency-independent with environment-specific magnitude.  
\end{itemize}

The closed-loop system interconnection in the so-called $\Delta-\mathcal{P}-K$ structure, which is the general form of the $\mathcal{H}_\infty$ design, is depicted in Figure \ref{fig:hinfsyn_DPK}. $c^T = [1, \; 0]$ is responsible for selecting $Q_1(t)$. $P_1$, and $P_2$ denote dynamical systems with only the first and second input channel of the nominal plant $P$, respectively. By applying the weighting and the compensator, the augmented plant $\mathcal{P} = \begin{bmatrix} \mathcal{P}_{11} & \mathcal{P}_{12} \\ \mathcal{P}_{21} & \mathcal{P}_{22} \end{bmatrix}$ can be formalized as
\begin{equation}
    \left[
    \begin{array}{c}
    y(t) \\
    z_u(t) \\
    z_p(t) \\ \hline
    e(t) \\
    \tilde y(t)
    \end{array}
    \right] = 
     \left[
     \begin{array}{c c c | c}
     0 & W_rP_1 & 0 & P_2 \\
     0 & 0 & 0 & W_u \\
     -c^T W_\Delta W_p & W_r(1-c^T P_1) W_p & W_{\hat Q}W_p & -c^T P_2 W_p  \\ \hline
     -c^T W_\Delta & W_r(1-c^T P_1) & W_{\hat Q} & -c^T P_2 \\
     W_\Delta & W_r P_1 & 0 & P_2
     \end{array}
     \right]
     \left[
    \begin{array}{c}
     d(t) \\
     r(t) \\
     \hat Q_2(t) \\ \hline
     u(t)
    \end{array}
    \right].
\end{equation}
\begin{figure}[htb!]
\centering
\includegraphics[width=0.5\linewidth]{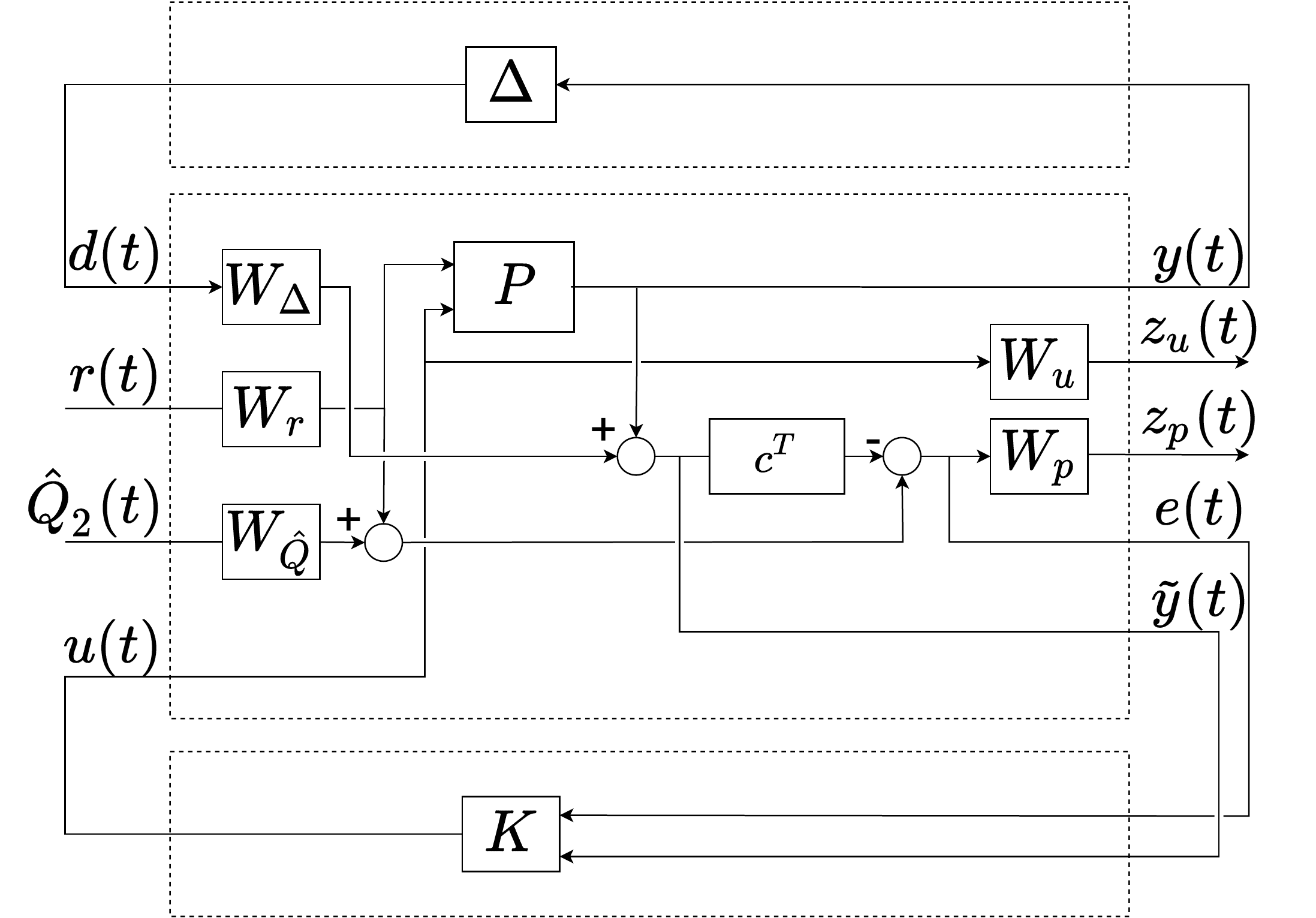}
\caption{Generalized $\Delta-\mathcal{P}-K$ structure}
\label{fig:hinfsyn_DPK}
\end{figure}
The closed-loop transfer function from the exogenous signals to the performance outputs can be expressed provided that the inverse $(I - \mathcal{P}_{22}K)^{-1}$ exists via a lower LFT as:
\begin{equation}
    \begin{bmatrix}
    y(t) \\ z_u(t) \\ z_p(t)
    \end{bmatrix} =
    \mathcal{F}_l(\mathcal{P}, K)
    \begin{bmatrix}
    d(t) \\ r(t) \\ \hat Q_2(t)
    \end{bmatrix},
\end{equation}
where
\begin{equation}
     \mathcal{F}_l(\mathcal{P}, K) = \mathcal{P}_{11} + \mathcal{P}_{12}K (I - \mathcal{P}_{22}K)^{-1} \mathcal{P}_{21}.
\end{equation}
In $\mathcal{H}_\infty$ control, the aim is finding a controller $K$ that minimizes the impact of the disturbance on the performance output. the the induced (worst-case) norm:
\begin{equation}
\underset{K}{min}||\mathcal{F}_l(\mathcal{P}, K)||_\infty < \gamma,
\end{equation}
where $\gamma$ is a prescribed disturbance attenuation level, progressively lowered by iteration
\cite{zhou1998essentials, skogestad2007multivariable}\footnote{Here $\gamma$ denotes a norm, not to be confused with the discount factor.}
. 

The resulting controller is an LTI system with 3 inputs  $u_c(t) \in \mathbb{R}^3$ ($\tilde y(t) \in \mathbb{2}$, $e(t)$) and one output $u(t)\in\mathbb{R}^1$. The synthesized controller is given in state-space form as 
\begin{equation*}
\dot x_c(t) = A_c x_c(t) + B u_c(t),
\end{equation*}
\begin{equation}
u(t) = c_c^T x_c(t) + D_c u_c(t),
\end{equation}
where the dynamical controller has several internal states $x_c(t)$. The controller is characterized by the matrices $A_c$ (state matrix), $B_c$ (input matrix, $c_c^T$ (output matrix), and $D_c$ (feed-through matrix). Sizes of these matrices depend on the number of internal states of the controller which is the result of the iterative control design process \cite{zhou1998essentials}.
When employing the controller in the learning context, the internal states are reset after each episode.

Next, the controlled learning loss is computed with the same assumptions as for the $\mathcal{H}_2$ case. The controlled evolution of $Q_1(t)$ is 
\begin{equation}
\label{eq:controlled_Q1_hinf}
    \frac{\partial Q_1(t)}{\partial t} = \Theta_1 \left( r(t) + \gamma Q_2(t) - Q_1(t) + u(t)\right).
\end{equation}
If the target is independent of $\theta$ and the control signal $u(t)$ only indirectly depends on the states of the learning agent, the controlled quadratic loss can be written as 
\begin{equation}
\label{eq:loss_Hinf}
    \mathcal L_{\mathcal{H}_\infty} = \frac{1}{2}\left(r(t) + \gamma\underset{a}{\mathrm{max}} Q(s(t+1),a, \theta) - Q(s(t), a(t), \theta) + u(t)\right)^2
\end{equation}
by using the chain-rule on Eq.~\eqref{eq:controlled_Q1_hinf} and integrating it w.r.t. $\theta$. 

\subsection{Fixed-structure output-feedback $\mathcal{H}_\infty$ controller design}
\label{sec:Hinf_fixed_control}
With the previously presented method, a dynamical $\mathcal{H}_\infty$ controller can be achieved in a general LTI structure. With the same technique, it is possible to optimize the free parameters of a fixed-structure controller. I.e.~define a structure for $K$ with tuneable parameters. This tuning minimizes the $\mathcal{H}_\infty$ norm of the closed-loop transfer function, resulting in a suboptimal controller. On the other hand, learning is in the low-frequency range (Figure \ref{fig:FFT}). Therefore, a static controller in place of a dynamic one should perform identically well. 
In the fixed-structure $\mathcal{H}_\infty$ controller design,  output-feedback gains $k_1$, $k_2$  and $\frac{k_3}{s}$ are sought, such that the closed-loop system is stable and the performance outputs (tracking error ($z_p(t)$), control energy ($z_u(t)$)) are minimized. Thus, the output $u(t)$ of the fixed-structure output-feedback $\mathcal{H}_\infty$ controller is
\begin{equation}
    u(t) =  - k_1 Q_1(t) - k_2 Q_2(t) + k_3 x_e(t).
\end{equation}
The loss function for the $\mathcal{H}_\infty$ controlled learning can be achieved with the same steps as for the $\mathcal{H}_2$ controller (Eqs.~\eqref{eq:Q_change_LQ}, \eqref{eq:theta_change_LQ}, \eqref{eq:loss_LQ}), and its structure will be similar too: 
\begin{align}
    \mathcal L_{\mathcal{H}_{\infty,f}} =  \frac{1}{2(1 + k_1)}\Big(  &  (r(t) + (\gamma - k_2) \underset{a}{\mathrm{max}} Q(s(t+1),a, \theta) - (1 +  k_1 )Q(s(t), a(t), \theta) + k_3 x_e(t)
 \Big)^2.
\end{align}
The difference is the way how the controller gains are achieved. The gains in the fixed-structure $\mathcal{H}_\infty$ controller consider the uncertainties implicitly. Thus they do not have to be recomputed every step. Therefore, it combines the best of two worlds: the simplicity of the $\mathcal{H}_2$ controller and the robustness of the dynamical $\mathcal{H}_\infty$ controller. On the other hand, it is suboptimal compared to the dynamical $\mathcal{H}_\infty$ controller.

\begin{remark}
\textbf{Tuneable parameters, heuristics, and transparency.} Reinforcement learning falls into the category of heuristics. In classical RL methods, the effect of modifying a tuneable parameter on learning is often unclear (e.g., NN structure, learning rate, replay buffer size, etc.). With the robust control approach it is possible to make such parameter tuning more transparent and procedural.

The NTK-based prediction of Q-value change is only valid for shallow and wide NNs with a low learning rate. This narrows down the choice of the function approximator. Value of the NTK (characterizing the nominal model and the parametric uncertainty) is determined by two factors: the structure of the NN and the magnitude of the environment states. In addition, the control-oriented tuning methodology is transparent and has well-established literature, making weight selection straightforward. Regardless of the environment, learning is in the low-frequency range, making the selection of $W_u$ and $W_p$ easier. The constant weights $W_\Delta$, $W_r$, and $W_{\hat Q}$ are used to normalize the input signals. 

The absence of target network and randomized replay memory abolishes some additional heuristics. 
\end{remark}

\section{Experiments}
\label{sec:experiments}
The three algorithms are tested on three environments from the OpenAI Gym with an increasing complexity: Cartpole \cite{barto1983neuronlike}, Acrobot \cite{sutton1996generalization}, and Mountain car \cite{brockman2016openai} (Figure \ref{fig:envs}). On each domain, the proposed algorithms are benchmarked against Double deep Q-learning \cite{mnih2015human}. 
In every experiment, the learning agent is a 2 layer fully-connected ReLU network with 2500 neurons with bias terms and appropriate input-output sizes. That is to comply with the assumptions in \citep{jacot2018neural}, i.e.~a shallow and wide neural network. The learning rate is $\alpha = 0.00005$, and the discount factor $\gamma = 1$.  In the following subsections, each environment will be described in detail, and the performance of the agents will be evaluated.
\begin{figure}[htb!]
\centering
\includegraphics[width=0.95\linewidth]{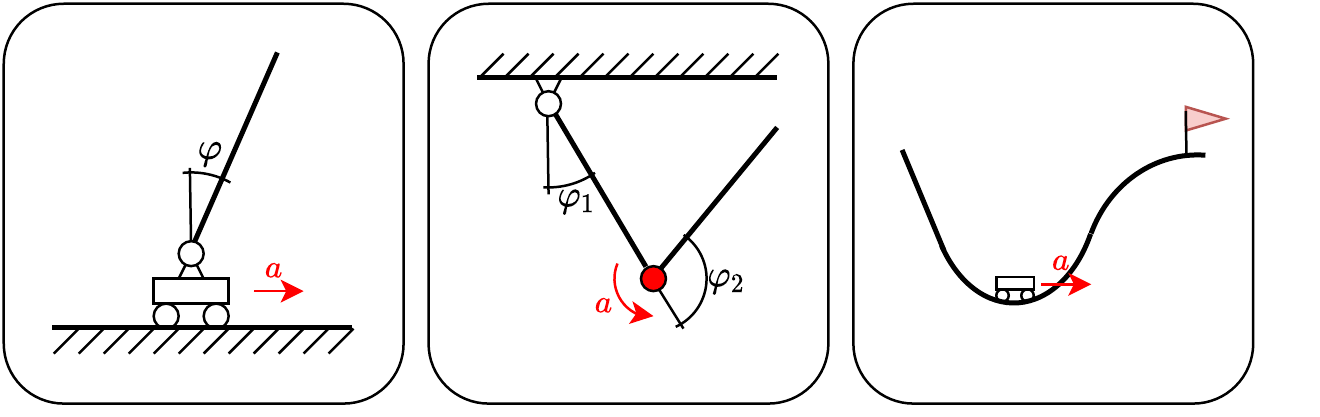}
\caption{Illustration of the environments. Left: Cartpole, middle: Acrobot, right: Mountain car.}
\label{fig:envs}
\end{figure}

\subsection{Cartpole}
\label{sec:cartpole}
The cartpole problem (also known as the inverted pendulum) is a common benchmark in control theory. The goal is to balance a pole to remain upright by horizontally moving the cart. The agent in this environment can take two actions: accelerating the cart left ($a^0$) or right ($a^1$). The state-space is characterized by four features: the position of the cart, its velocity, the pole angle, and the pole angular velocity. In the reinforcement learning setting, the agent's goal is to balance the pole as long as possible. A +1 reward is given for every discrete step if the pole is in vertical direction, and the episode ends if the pole falls or successfully keeps balancing for 200 steps. OpenAI defines the pass criteria for this Gym environment as an average reward of 195 for 100 episodes.

Figure \ref{fig:Cartpole_learning} depicts the convergence of the deep Q-learning augmented with three different controllers in the cartpole environment. Both algorithms can solve the environment and eventually reach the target moving average reward of 195. The dynamic $\mathcal{H}_\infty$ converges the fastest with the least standard deviation, followed by the fixed-structure $\mathcal{H}_\infty$ controller (denoted by $\mathcal{H}_{\infty,f}$ in the plots), and the DDQN. The $\mathcal{H}_2$-controlled agent gets stuck in a local optimum. 
While designing the $\mathcal{H}_2$ controller, an interesting observation was made. The diagonal elements of weighting matrix $W_x$ penalize $Q_1(t)$, $Q_2(t)$, and $x_e(t)$, respectively. The best result could be achieved when the second element, i.e.,~the one penalizing the magnitude of $Q_2(t)$ was high. This suggests, while maintaining stability and tracking, the change of $Q_2(t)$ (the target) was slowed down, yielding a similar (but continuous) behavior to learning with a target network. Despite its good performance, the $\mathcal{H}_2$ controller has obvious drawbacks: learning is stabilized only in a local, one-step-ahead sense. In addition, the NTK has to be recomputed after every step, making it computationally intensive. On the other hand, it eliminates the need to give a bound for $\Theta_1$, and $\Theta_2$, which is crucial for the two robust approaches. In addition, the standard deviation decreases, and oscillations are eliminated as the learning progresses. This is not the case for DDQN, which is known to be prone to oscillations \cite{ruder2016overview}. 
\begin{figure}[htb!]
    \centering
    \includegraphics[width=0.5\linewidth]{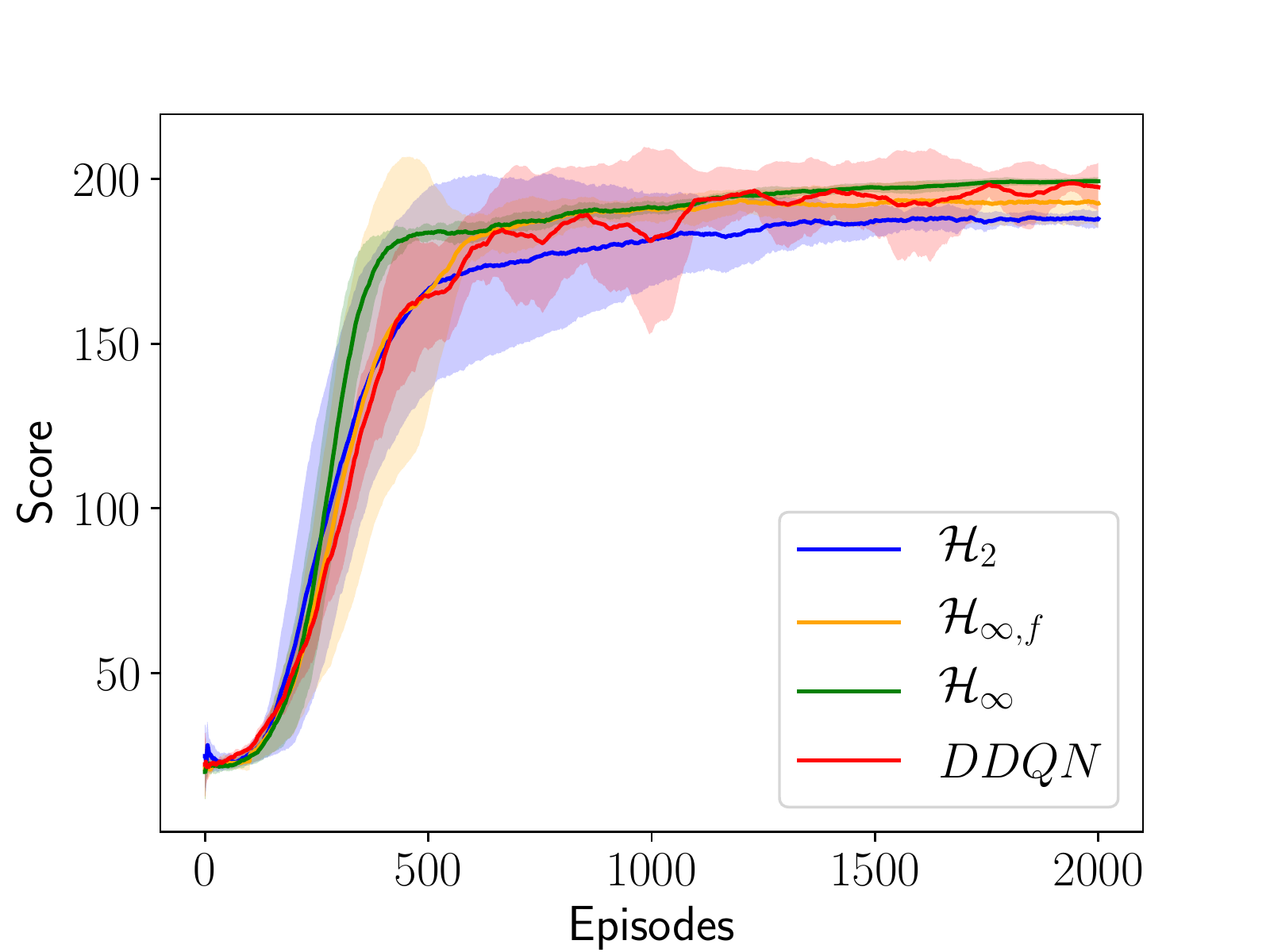}
    \caption{Learning with different methods in the Cartpole environment, average of 10  random seeds.}
    \label{fig:Cartpole_learning}
\end{figure}

Next, an in-depth analysis of the $\mathcal{H}_\infty$-controlled learning is performed in terms of loss, $Q$ value evolution (tracking), control input, and the range of the parameter variation. Naturally, as the agent's policy converges to the optimal one, the loss decreases, Figure \ref{fig:Cartpole_loss}. Since the NN is overparametrized, the loss will become near-zero. However, there are some peaks: the loss can become very high when the agent encounters a previously unvisited state (via exploration or a near-failure state). That is because the agent does not know what to do and takes a wrong action. 
\begin{figure}[htb!]
    \centering
    \includegraphics[width=0.5\linewidth]{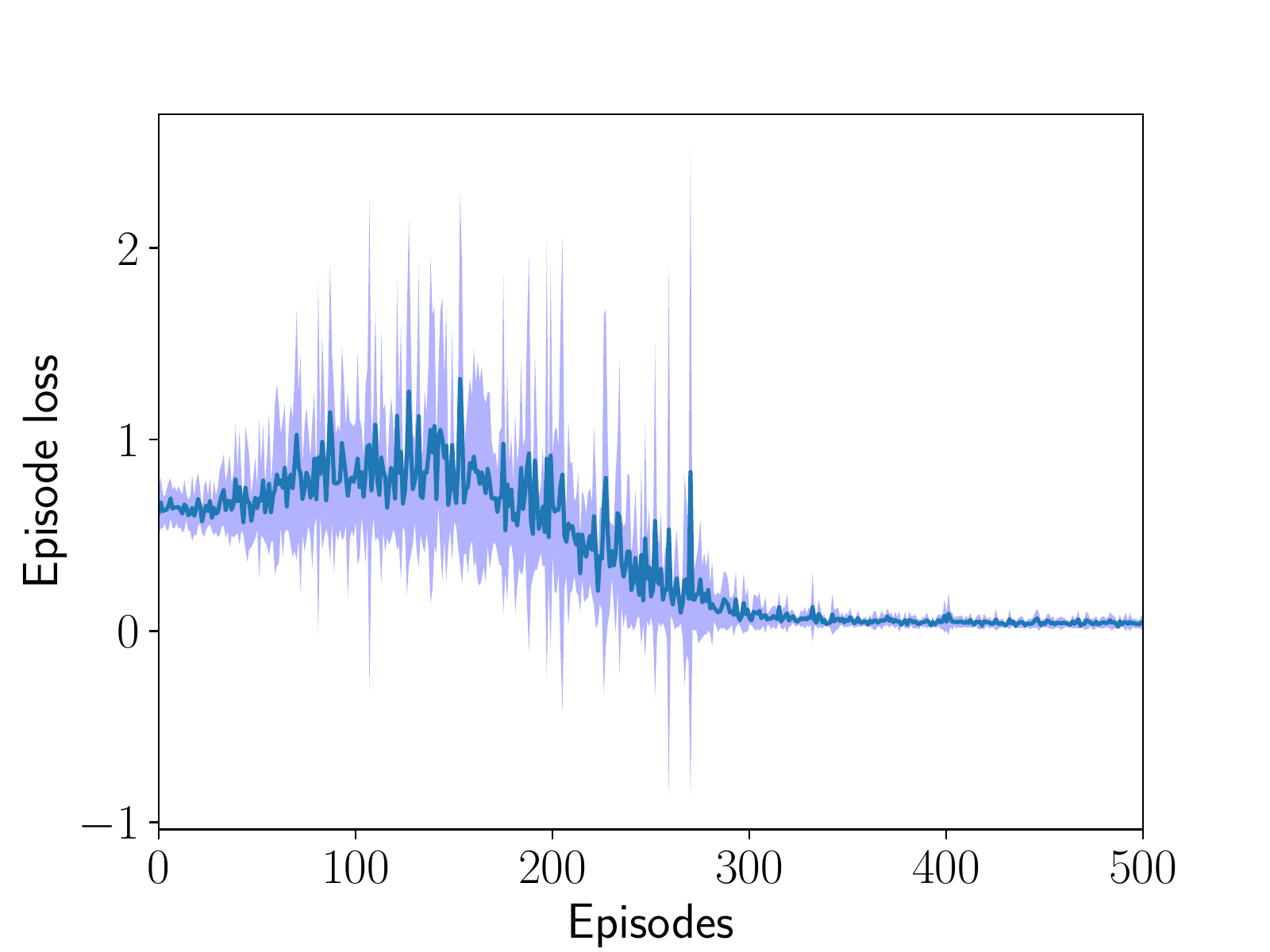}
    \caption{Loss in in the Cartpole environment with $\mathcal{H}_\infty$ control, average of 10 random seeds.}
    \label{fig:Cartpole_loss}
\end{figure}
$Q$ values are the states of the learning model. Figure \ref{fig:Cartpole_Q} depicts the convergence of $Q_1(t)$ and $Q_2(t)$. It can be observed that Q-values converge to a much lower value than what one would expect in a tabular Q-learning case. Since the maximum cumulated reward is 200, according to Eq.~\eqref{eq:Qdef} the true Q-values for an optimal state-action pair should be 200 (assuming $\gamma = 1$). On the other hand, it is well known that DQN tends to overestimate Q-values. The Q-values for the DDQN benchmark converge to $\approx$295. The bias in the controlled learning stems from two sources: one is the smoothing property of function approximation.  $Q_1(t)$ and $Q_2(t)$ are very close to each other, confirming the smoothness of the Q-function.
Additionally, the controller integrates the tracking error. This has a substantial smoothing property on the Q-function too. In addition, the control input $u(t)$ acts as a dynamically changing bias in the loss (Eq.~\eqref{eq:loss_Hinf}).
Zooming in to one episode (Figure \ref{fig:Cartpole_Q1}), Only some minor oscillations of $Q_1(t)$ could be seen. That is because $Q_1(t)$ equals $Q_2(t)$'s previous value if the agent follows a greedy policy. At the beginning of the episode, it is not the case because the control input is high at that time which makes the two $Q$ values drift from each other, see Figure \ref{fig:Cartpole_controlinput1}. The control input (Figure \ref{fig:Cartpole_controlinput}) has peaks at the start of each episode. It is due to the reset of the controller's internal states. Note that the control input peaks are in the same magnitude range as the actual Q-values.
Finally, Figure \ref{fig:Cartpole_theta} shows the NTK values during learning, with the bounds prescribed in Figure \ref{fig:NTK_range}. It is outside the prescribed range in the first few episodes, where exploration is more significant. After the learning it starts converging, $\Theta_1$ values remain within the bounds, used for robust controller design. 
\begin{figure}[htb!]
\centering
\begin{subfigure}{.45\textwidth}
  \centering
    \includegraphics[width=0.99\linewidth]{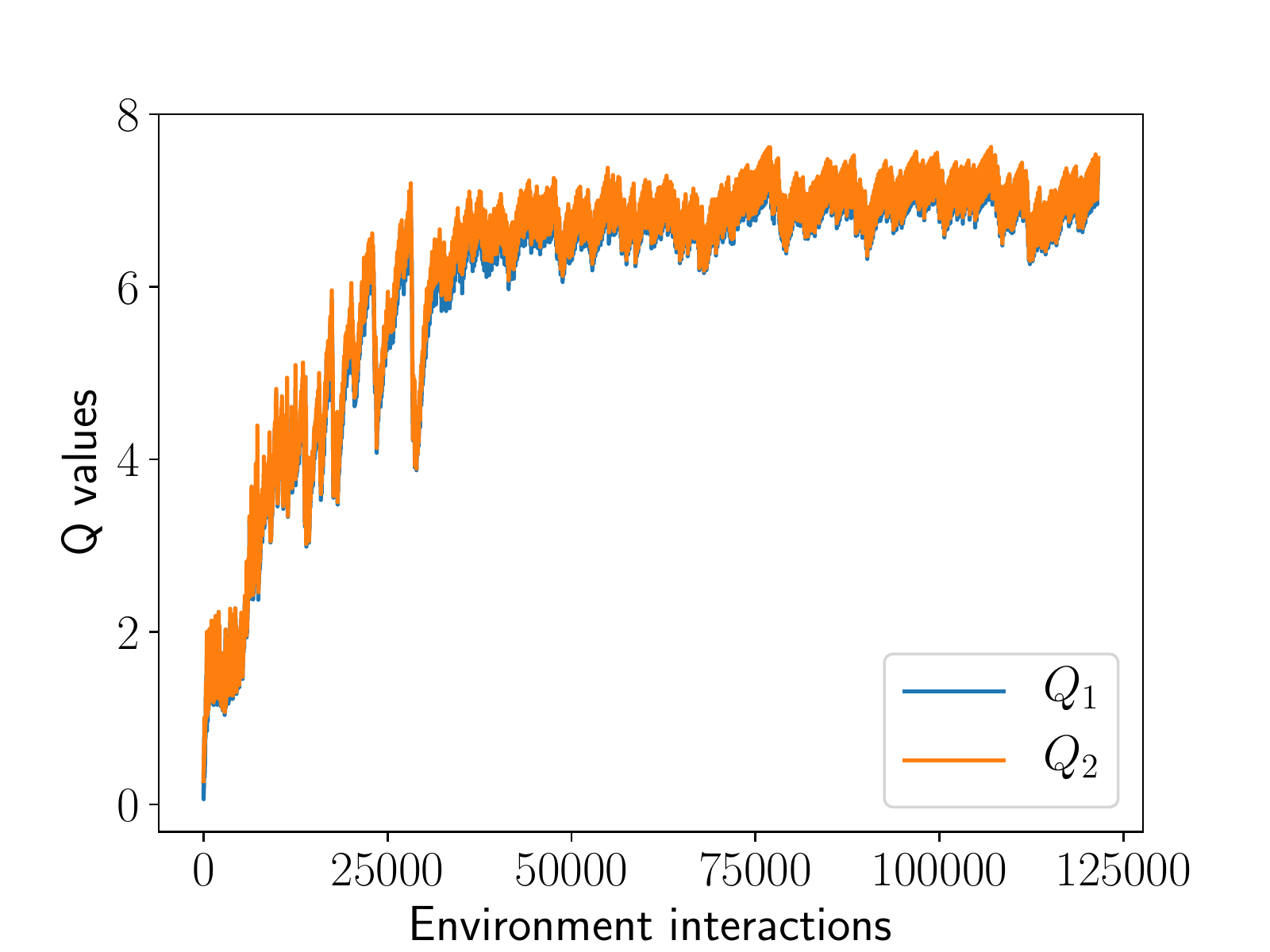}
    \caption{500 episodes.}
    \label{fig:Cartpole_Q500}
\end{subfigure}%
\begin{subfigure}{.45\textwidth}
  \centering
    \includegraphics[width=0.99\linewidth]{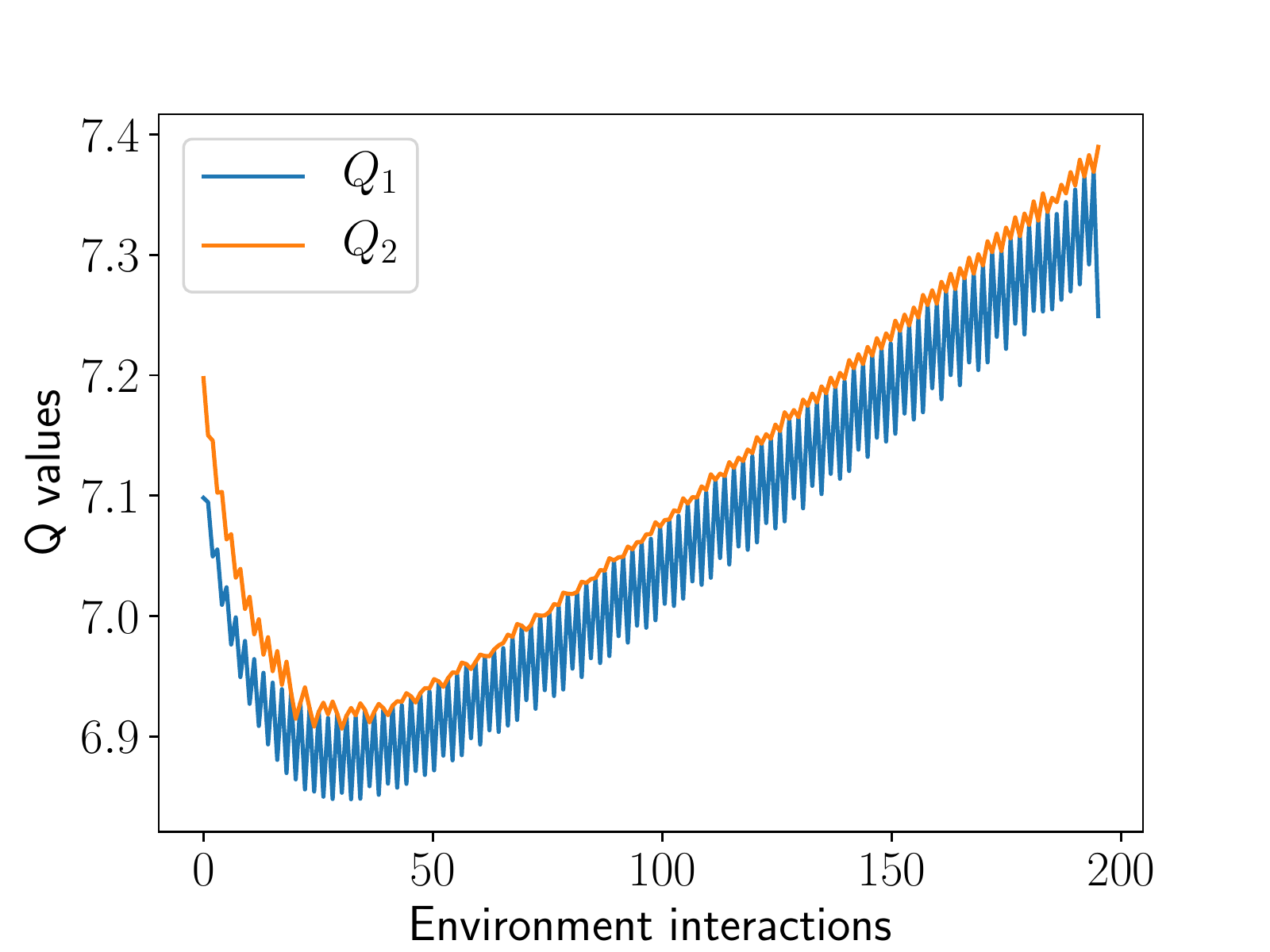}
    \caption{One episode.}
    \label{fig:Cartpole_Q1}
\end{subfigure}
\caption{Q-values in the Cartpole environment with $\mathcal{H}_\infty$ controller.}
\label{fig:Cartpole_Q}
\end{figure}

\begin{figure}[htb!]
\centering
\begin{subfigure}{.45\textwidth}
  \centering
    \includegraphics[width=0.99\linewidth]{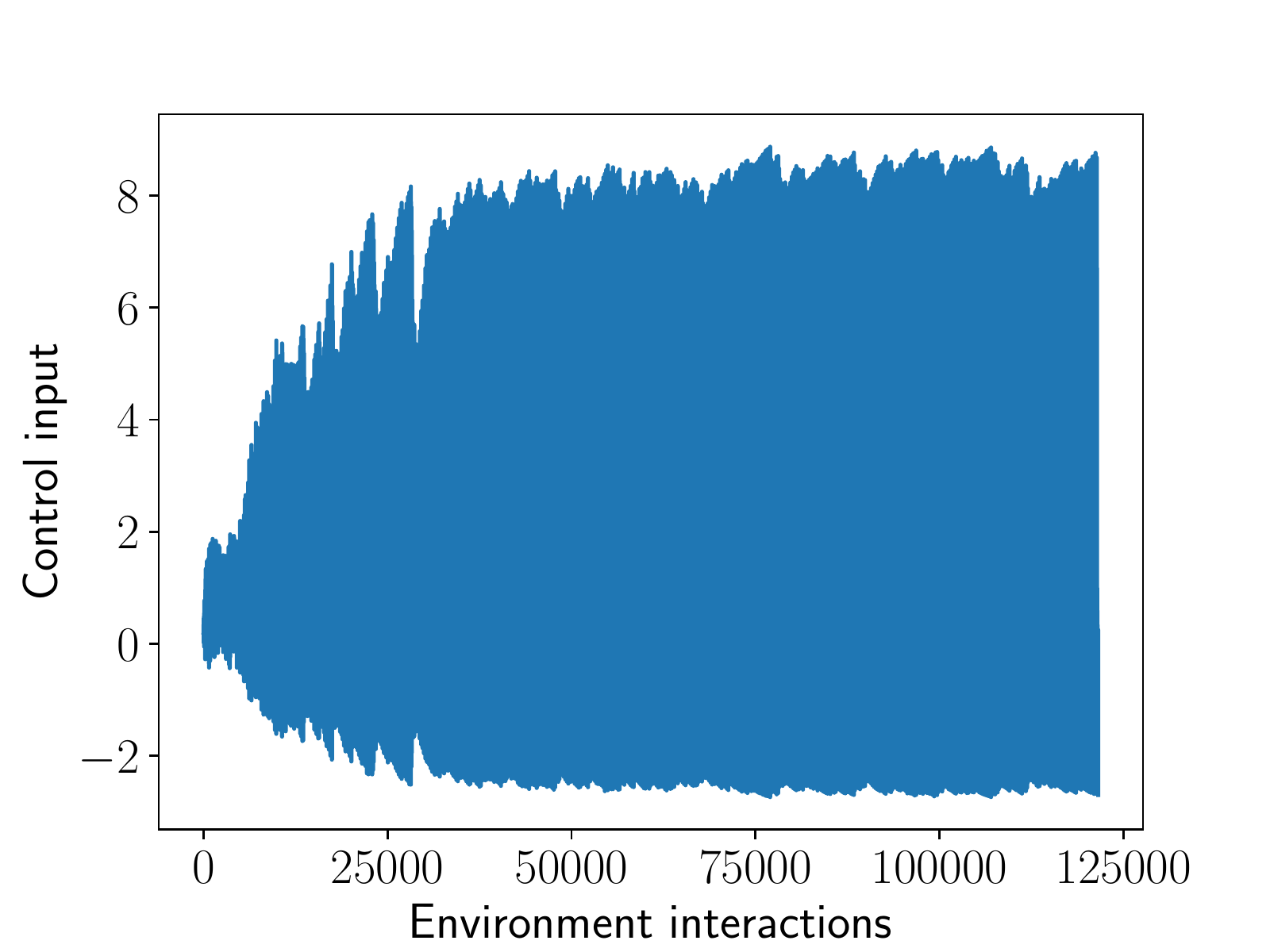}
    \caption{500 episodes.}
\end{subfigure}%
\begin{subfigure}{.45\textwidth}
  \centering
    \includegraphics[width=0.99\linewidth]{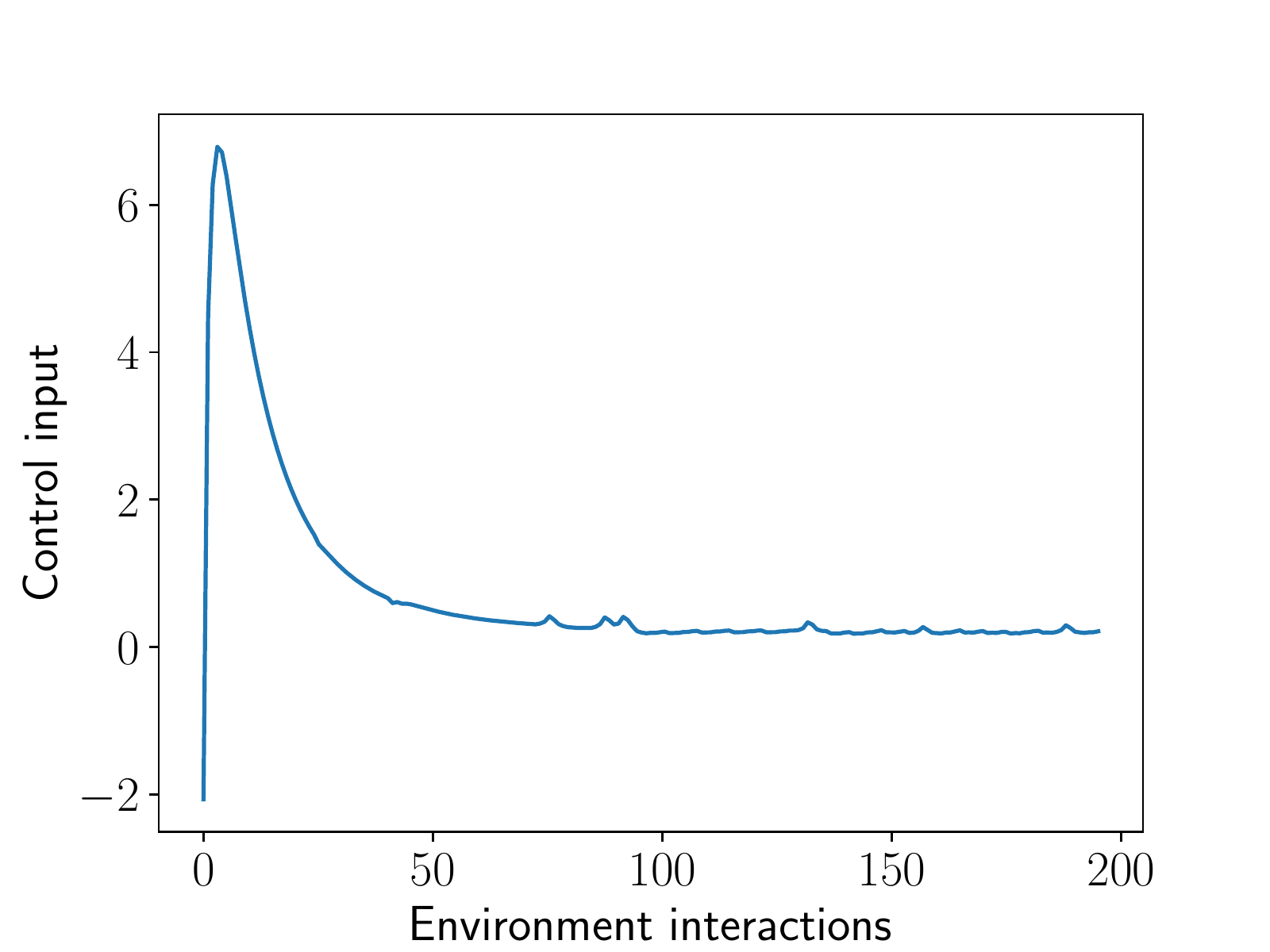}
    \caption{One episode.}
    \label{fig:Cartpole_controlinput1}

\end{subfigure}
\caption{Control input $u(t)$ in the Cartpole environment with $\mathcal{H}_\infty$ controller.}
\label{fig:Cartpole_controlinput}
\end{figure}

\begin{figure}[htb!]
    \centering
    \includegraphics[width=0.5\linewidth]{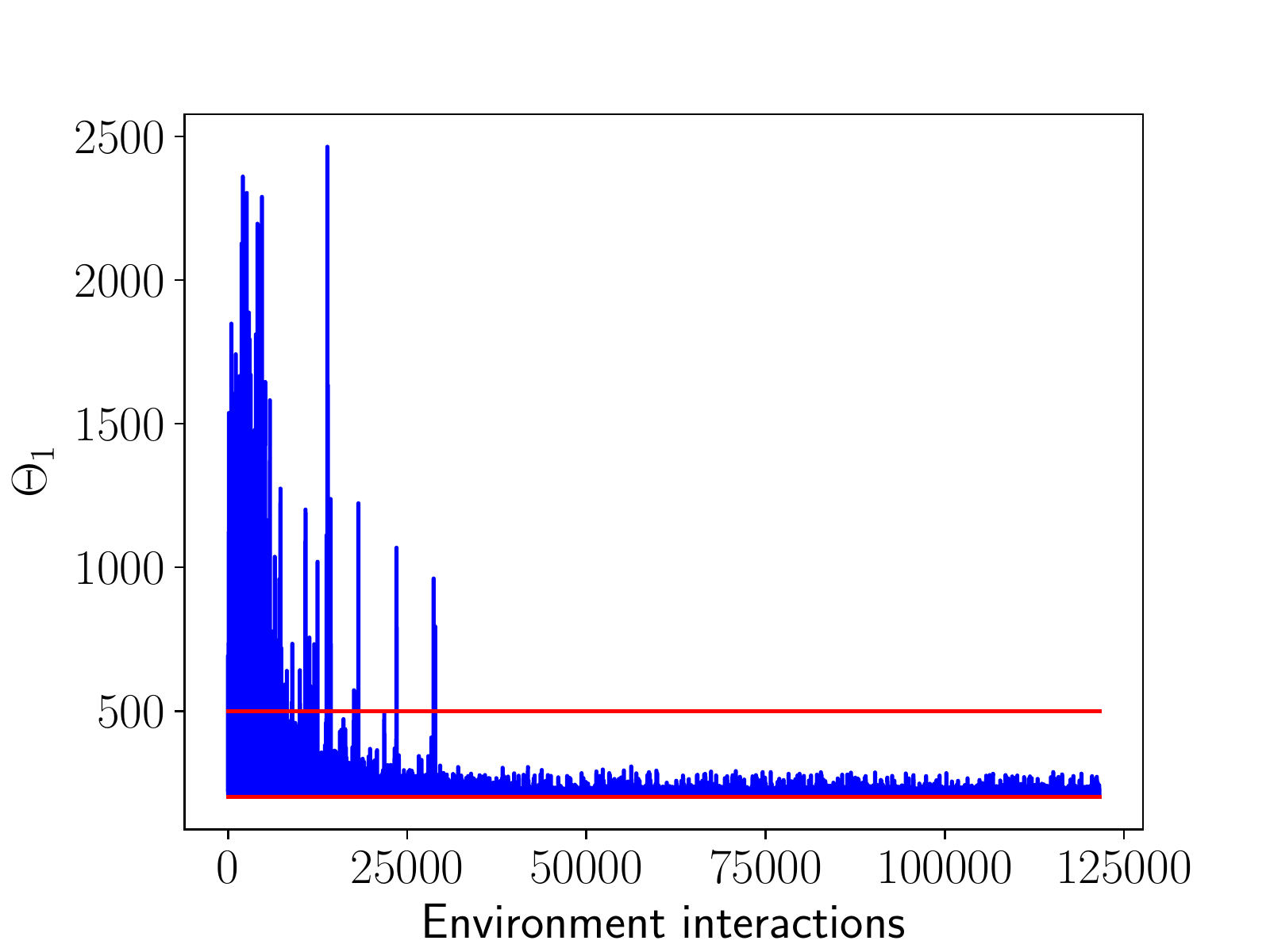}
    \caption{$\Theta_1$ values in the Cartpole environment with $\mathcal{H}_\infty$ controller.}
    \label{fig:Cartpole_theta}
\end{figure}


\subsection{Acrobot}
\label{sec:acrobot}
This environment is slightly more complex than the Cartpole. It is an under-actuated dynamical system where the goal is swinging up a double pendulum using an actuated joint between the two links. The environment has six continuous states (Cartesian coordinates of the pole ends and their angular velocities) and three discrete actions (positive or negative torque, or none). 
The episode is successful if the lower part reaches a certain height. For every step spent trying to swing the pole up, the agent gets $-1$ reward.

In this environment, the controlled agents are on par with DDQN; see Figure \ref{fig:Acrobot_learning} except for the $\mathcal{H}_2$-controlled one. It gets stuck in a local optimum in one of the random seeds, pulling the average score down. Oscillations are not significant in either case.
\begin{figure}[htb!]
    \centering
    \includegraphics[width=0.5\linewidth]{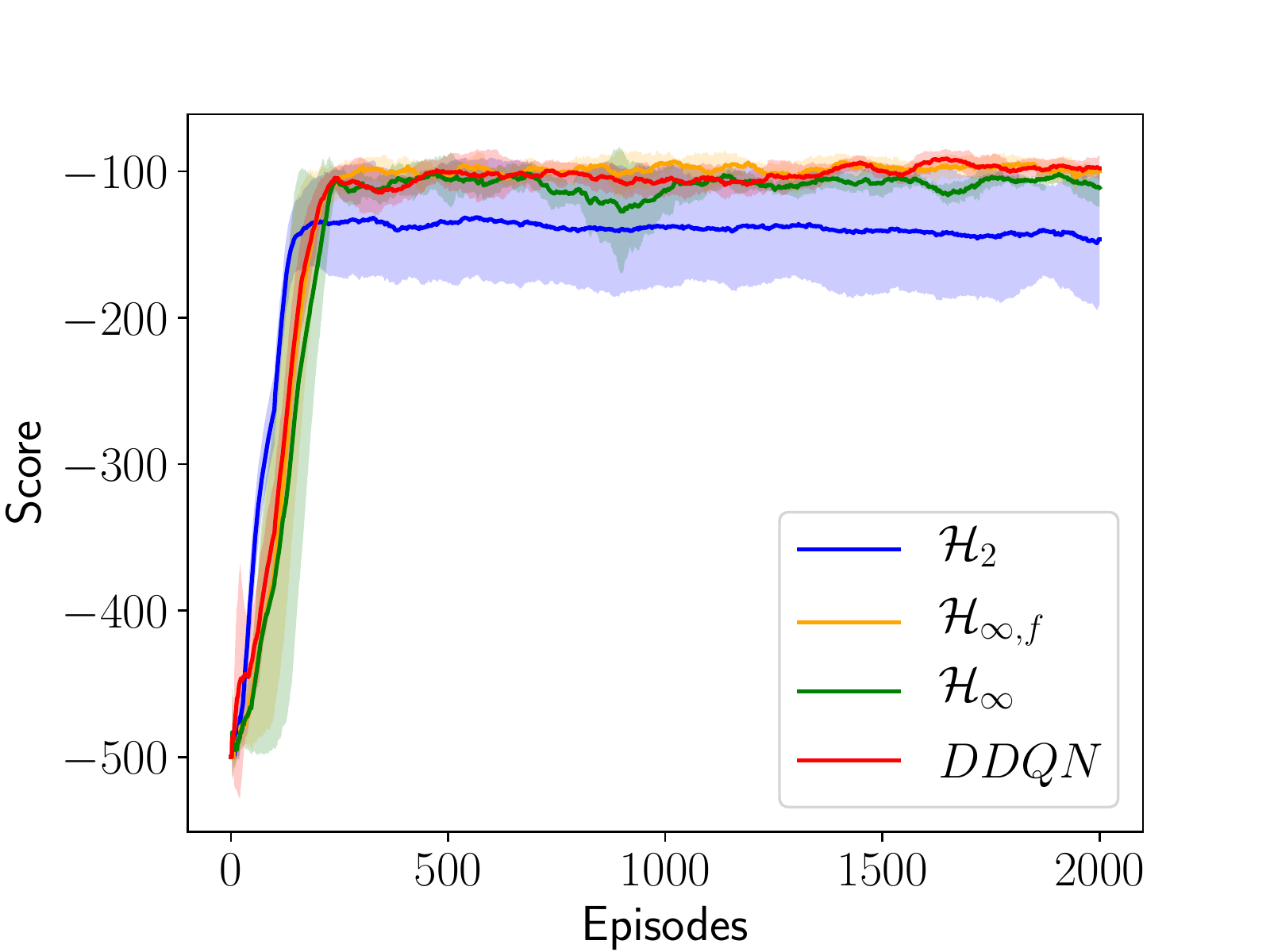}
    \caption{Learning with different methods in the Acrobot environment, average of 10  random seeds.}
    \label{fig:Acrobot_learning}
\end{figure}

\subsection{Mountain car}
\label{sec:mountaincar}
The Mountain car environment is the epitome of a sparse reward environment. "A car is on a one-dimensional track, positioned between two "mountains." The goal is to drive up the mountain on the right; however, the car's engine is not strong enough to scale the mountain in a single pass. Therefore, the only way to succeed is to drive back and forth to build up momentum. \cite{brockman2016openai}" Although it has small state-space (longitudinal position, and velocity of the car) and action space (accelerate to the left, right, or none), the agent only gets $-1$ reward every step it spends in the environment. The episode terminates if the car reaches its goal on top of the mountain or fails for 200 steps. Therefore, successful exploration is the key in this environment: the agent must find the top of the mountain via exploration as soon as possible.  
\begin{figure}[htb!]
    \centering
    \includegraphics[width=0.5\linewidth]{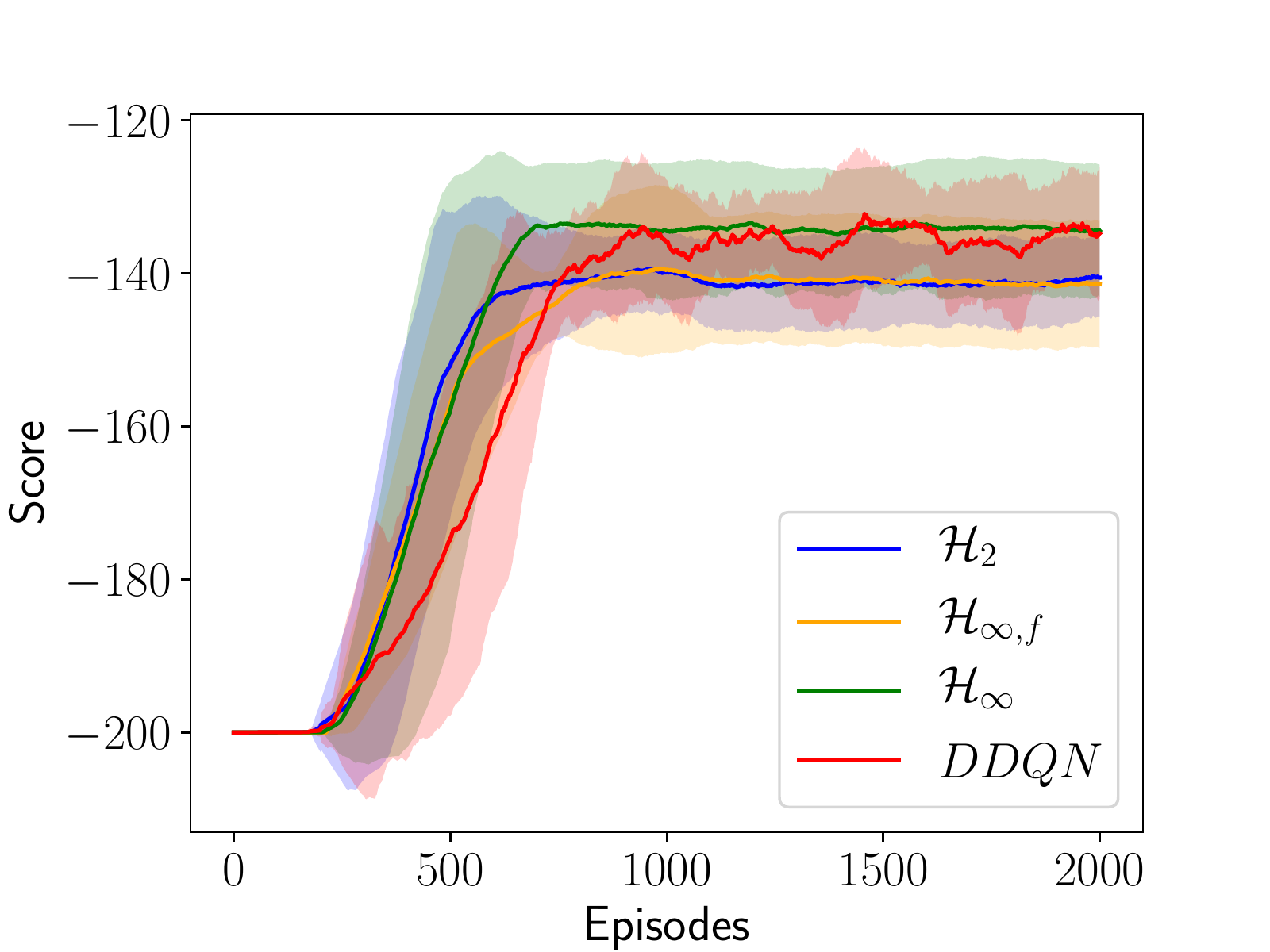}
    \caption{Learning with different methods in the Mountain car environment, average of 10  random seeds.}
    \label{fig:Mountaincar_learning}
\end{figure}
The proposed cascade controllers can ensure convergence, but they do not interfere with the exploration strategy. Therefore, the controlled agents start getting better scores at the same time as the benchmark. The controlled methods reach a plateau faster. On the other hand, this plateau is lower for the $\mathcal{H}_2$ and $\mathcal{H}_\infty$ controllers than the DDQN. The dynamical $\mathcal{H}_\infty$ controller is smoother and reaches a similar average score to the benchmark.

\section{Conclusions}
\label{sec:conclusions}
In this paper, deep Q-learning has been described as an uncertain LTI system, relying on some properties of the NTK, and assuming a deep and shallow neural network. This dynamical approach enables tackling deep Q-learning from a different perspective. Instead of employing random experience replay or a target network, cascade stabilizing controllers were formulated. For controller design, first, the magnitude of uncertainties (NTK parametric variation, uncertainty in the states, and exploration) have been evaluated. Then, the input and output signals in frequency domain have been analyzed. Simulations concluded that signals are in the low-frequency domain, and uncertainties can be bounded. 
Therefore, three controllers: $\mathcal{H}_2$, $\mathcal{H}_\infty$, and fixed-structure $\mathcal{H}_\infty$ have been proposed. The $\mathcal{H}_2$ controller cannot handle uncertainties and must be recomputed every step, only guaranteeing local stability. The $\mathcal{H}_\infty$ controller is developed in frequency domain considering uncertainties too. Low-frequency of the signals makes it possible to synthesize a controller with constant gains that has matching performance to the dynamical $\mathcal{H}_\infty$ controller. The integrating property of the proposed cascade controllers has a smoothing effect on the Q-function, acting as a bias in the loss function. Therefore, learned Q-values will be offset from the theoretical (tabular) ones.
On the other hand, the proposed approach requires fewer heuristics and provides more transparency. Assumptions for the NTK and the control-oriented weighting make the agent's design more straightforward. In addition, the absence of the target network and randomized replay memory obviates further the need for heuristics. 
The synthesized controlled learning methods were tested in three OpenAI Gym environments. Results are summarized numerically in Table \ref{tab:result_summary}. It can be concluded that the dynamical $\mathcal{H}_\infty$ is outperforms the DDQN in Cartpole and Mountain car, while it fails in the Acrobot environment. The non-dynamical ($\mathcal{H}_2$ controller and the fixed-structure $\mathcal{H}_\infty$ controller with limited dynamics cannot compete with the benchmark. In terms of environment interactions, learning can be sped up significantly, while guaranteeing convergent behaviour with the controlled loss functions. 
\begin{table}[htb] 
\centering
\begin{tabular}{lrrr}
\hline
& \multicolumn{1}{c}{Cartpole} & \multicolumn{1}{c}{Acrobot} & \multicolumn{1}{c}{Mountain car}   \\ \hline
$\mathcal{H}_2$ & 187.92 $\pm$ 2.47  & -146.44 $\pm$ 45.09   & -140.57 $\pm$ 5.07  \\
$\mathcal{H}_{\infty,f}$ & 192.59 $\pm$ 7.13  & -99.93 $\pm$ 3.79 & -141.41 $\pm$ 8.38  \\
$\mathcal{H}_{\infty}$   & \textbf{199.37 $\pm$ 1.21} & -111.73 $\pm$ 13.17   & \textbf{-134.45 $\pm$ 8.63} \\
DDQN & 197.56 $\pm$ 7.37  & \textbf{-98.04 $\pm$ 8.58} & -135.79 $\pm$ 9.41 
\end{tabular}
\caption{Final scores (after 2000 episodes of learning) in every environment. Best scores are highlighted in bold.}
\label{tab:result_summary}
\end{table}

\section{Acknowledgements}
This work has been supported and funded by the project RITE (funded by CHAIR, Chalmers University of Technology). The authors would like to thank Vincent Szolnoky and Viktor Andersson for their insightful comments on the paper. 

\section{Data Availability Statement}
The data and source code files that support the findings of this study are openly available in the GitHub repository Controlled\_DQN at \url{https://github.com/bva-bme/Controlled\_DQN}, reference number \cite{bva_github}.

\bibliography{biblio}
\end{document}